\documentclass[letterpaper, 10 pt, conference]{ieeeconf}  
\IEEEoverridecommandlockouts                              
\usepackage{amsmath,amsfonts}
\usepackage{algorithm}
\usepackage{array}
\usepackage{textcomp}
\usepackage{stfloats}
\usepackage{url}
\usepackage{verbatim}
\usepackage{graphicx}
\usepackage{cite}
\usepackage{booktabs}
\usepackage{multirow}
\usepackage{tabularx}
\usepackage{subcaption}
\usepackage{algpseudocode}
\usepackage[
    letterpaper,
    top=60pt,
    bottom=43pt,
    left=48pt,
    right=48pt
]{geometry}

\newtheorem{proposition}{Proposition}

\begin{document}
\title{\LARGE \bf Monte Carlo-Based Indoor UAV Localization in Dynamic Environments Using Only Four Onboard Range Sensors}


\author{Mohammad Reza Ranjbar Divkoti and A. Pedro Aguiar 
\thanks{Authors are with SYSTEC-ARISE Research Center for Systems and Technologies, Faculty of Engineering, University of Porto 4200-465 Porto, Portugal.}%
\thanks{This work, conducted under the FCT doctoral scholarship with reference number 2023.01812.BD, was partially supported by R\&D Unit SYSTEC--Systems and Technologies Center (UID/00147, DOI 10.54499/UID/00147/2025) and Associate Laboratory ARISE--Advanced Production and Intelligent Systems (DOI 10.54499/LA/P/0112/2020), both funded by Fundação para a Ciência e a Tecnologia (FCT), through the national funds. This work has been submitted to the IEEE for possible publication. Copyright may be transferred without notice, after which this version may no longer be accessible.}
}

\maketitle
\thispagestyle{empty}
\pagestyle{empty}

\begin{abstract}
Indoor UAV localization with minimal onboard sensing remains challenging, particularly in dynamic environments. We propose augmented filtered Monte Carlo localization (AF-MCL) for 2D localization using only four single-point range sensors and measured heading. AF-MCL rejects range measurements inconsistent with the known static map and adaptively injects particles to recover from localization divergence and kidnapping. Beam classification is performed only at the previous position estimate, making its ray-casting cost independent of the particle count. The particle distribution additionally provides an online measure of localization uncertainty. Real-world experiments in a 2.10~$\times~$1.63~m environment yield mean position errors of 6.4--7.7~cm under static and dynamic conditions. Simulations in 6~$\times$~6~m and 10~$\times$~10~m environments yield 4.5--5.2~cm mean errors; under dynamic obstacles, standard MCL reaches 12.3~cm and reliable MCL 7.7~cm. The results show that robust indoor localization under dynamic obstacles is achievable with only four single-point range measurements and heading.
\end{abstract}

\section{Introduction}
Unmanned Aerial Vehicles (UAVs) are increasingly
deployed for inspection, inventory monitoring, environmental sensing,
and search-and-rescue missions. Their operation indoors requires
accurate and low-latency localization despite confined geometry,
clutter, and the absence of a reliable global positioning
signal~\cite{heidari2023machine,boursianis2022internet,
ferreira2017localization,wu2023flytracker}.

Existing indoor localization solutions typically trade sensing and
computational complexity for accuracy. UWB-based methods require
external infrastructure and remain sensitive to multipath and
non-line-of-sight conditions~\cite{zhao2022finding,hashim2023uwb,
nguyen2021range}, while vision- and LiDAR-based approaches provide rich
environmental information at higher sensing and computational
cost~\cite{akai2023reliable,he2022vision,lu2019l3,
zimmerman2022robust,nguyen2025uloc}. Visual-inertial methods reduce
dependence on external infrastructure but accumulate drift, motivating
fusion with additional sensing modalities~\cite{wu2023flytracker,
chen2024refloc,jiang2024robust}. These limitations motivate
localization using the sparse, lightweight range sensing available on
small UAV platforms.

A closely related approach is~\cite{muller2023fully}, which performs
onboard Monte Carlo localization on a nano-UAV using a pair of 64-zone
time-of-flight (ToF) sensors and a parallel ultra-low-power processor.
Although computationally efficient, each multizone sensor provides
substantially more spatial information than a single-point range
measurement, and the method is evaluated in static environments. The
problem considered here is therefore more sensing-constrained: can
reliable indoor localization be achieved from only four single-point
range measurements and heading when individual measurements may also
be corrupted by unknown moving obstacles?

We address this problem using four single-point range sensors oriented
front, back, left, and right. The UAV position is estimated relative
to a known initial position in a known static map, using the four range
measurements and measured heading $\theta$. Unknown dynamic obstacles
may enter the sensor beams and generate ranges inconsistent with the
static map, as illustrated in Fig.~\ref{fig:problem_def}. No external
localization infrastructure is used.

\begin{figure}[t]
    \centering
    \includegraphics[width=0.4\linewidth]{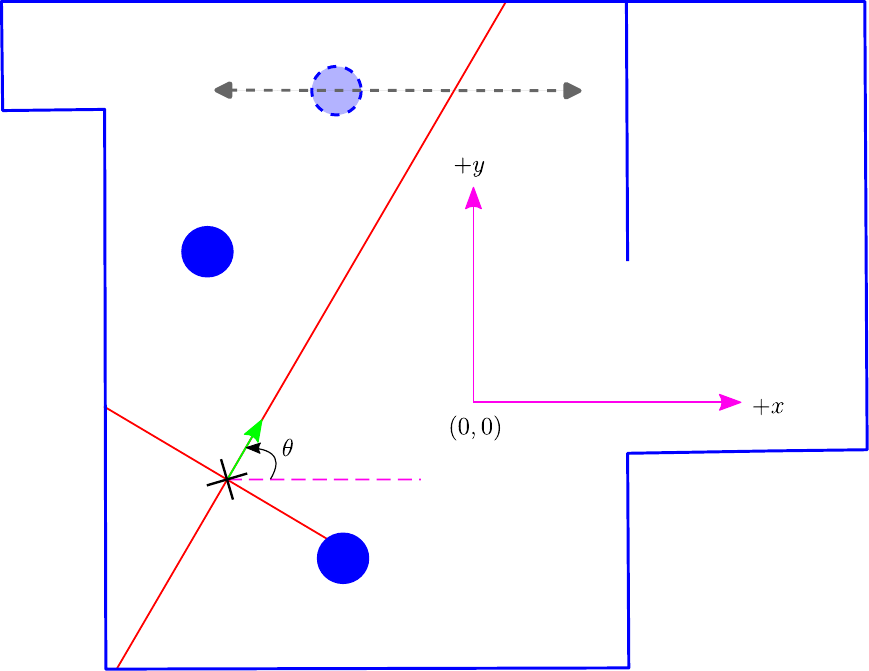}
    \caption{UAV localization with four directional range sensors in a
    known static map and unknown dynamic obstacles.}
    \label{fig:problem_def}
\end{figure}

To address this setting, we propose AF-MCL. AF-MCL identifies and rejects range measurements
inconsistent with the static map before particle weighting, and
adaptively injects particles when filter divergence is detected.
Crucially, beam classification is evaluated only at the previous
position estimate rather than independently at every particle,
eliminating particle-count dependence from the classification
ray-casting stage. The particle distribution additionally provides an
online measure of localization uncertainty.

The contribution is threefold. First, we demonstrate 2D indoor UAV
localization using only four single-point ranges and measured heading
under both static and dynamic obstacles. Second, AF-MCL combines
dynamic-beam filtering with adaptive particle injection for robustness
to corrupted measurements and localization failures, while requiring
only one classification ray cast per sensor, independent of the
particle count. Third, simulation and real-world experiments evaluate
the contributions of filtering and augmentation, recovery from
kidnapping, computational scaling, retained-beam geometry, and
localization performance. Mean errors of 4.5--5.2~cm are obtained in
the simulated $6\times6$~m and $10\times10$~m environments, while
real-world flights in a $2.10\times1.63$~m environment yield
6.4--7.7~cm mean error.

\section{Monte Carlo-Based Localization Algorithm}
\label{localization_algorithm}

This section presents the proposed Monte Carlo-based localization framework. It first introduces the standard Monte Carlo Localization (MCL) algorithm for static indoor environments, then presents an AF-MCL extension that improves robustness to unknown dynamic obstacles and localization failures. 

\subsection{Monte Carlo Localization for Static Indoor Environments}

State representation: Let the UAV's 2D pose at time step $t$ be defined as $\mathbf{x}_t = [x_t, y_t, \theta_t]^\top$, where $(x_t, y_t) \in \mathbb{R}^2$ is its global position. The yaw orientation $\theta_t \in [-\pi, \pi)$ is measured directly by an onboard sensor (e.g., an IMU) and treated as a known measurement rather than an estimated state variable.

The belief over the 2D position at time step $t$ is represented by a set of $N$ weighted particles $\mathcal{X}_t = \left\{ \mathbf{x}_t^{[i]}, w_t^{[i]} \right\}_{i=1}^{N}$. The state hypothesis of the $i$-th particle is given by $\mathbf{x}_t^{[i]} = [x_t^{[i]}, y_t^{[i]}, \theta_t]^\top$, where the position components $(x_t^{[i]}, y_t^{[i]})$ are estimated, while the orientation $\theta_t$ is not estimated and is set directly to the UAV's current measured yaw. Each particle has a normalized importance weight $w_t^{[i]} \ge 0$ such that $\sum_{i=1}^{N} w_t^{[i]} = 1$.

Particle Set initialization: Particles are sampled around the known initial position $\mathbf{\bar{x}}_0$ using a 2D Gaussian distribution $\mathcal{N}(\mathbf{\bar{x}}_0, \Sigma_0)$, where $\Sigma_0$ is the initial position covariance. All weights are set uniformly to $\frac{1}{N}$.

Bayesian formulation: Following the recursive Bayes filter formulation, at time step $t$, the belief distribution $\mathrm{bel}(\mathbf{x}_t)$ is updated as $\mathrm{bel}(\mathbf{x}_t) = \eta \, p(z_t \mid \mathbf{x}_t, m) \int p(\mathbf{x}_t \mid \mathbf{u}_t, \mathbf{x}_{t-1}) \, \mathrm{bel}(\mathbf{x}_{t-1}) \, d\mathbf{x}_{t-1}$, where $\eta$ is a normalization constant, $z_t$ is the sensor measurement vector at time step $t$, and $m$ represents the global occupancy grid map. The control input $\mathbf{u}_t = [v_x, v_y]_t^\top$ consists of the UAV's body-frame linear velocities at time step $t$.

Motion update: Each particle's position is updated by projecting body-frame velocities $\mathbf{u}_t$ into the global frame using measured yaw $\theta_t$, $\mathbf{x}_t^{[i]} = \mathbf{x}_{t-1}^{[i]} + R(\theta_t)\mathbf{u}_t \Delta t + \epsilon_t, \quad \epsilon_t \sim \mathcal{N}(\mathbf{0}, \Sigma_u \Delta t)$, where $R(\theta_t)$ is the 2D rotation matrix and $\epsilon_t$ is Gaussian noise with motion covariance $\Sigma_u$.

Measurement model: Let $q \in \{f, l, b, r\}$ index the four distance sensors (front, left, back, right) mounted on the UAV, where each sensor has a fixed mounting offset $\phi_q$ and yields a measured distance $z_t^q$ at time step $t$. For each sensor $q$, the beam's global endpoint $\hat{p}_q^{[i]} = \begin{bmatrix} x_t^{[i]} \\ y_t^{[i]} \end{bmatrix} + z_t^q \begin{bmatrix} \cos(\theta_t + \phi_q) \\ \sin(\theta_t + \phi_q) \end{bmatrix}$. The measurement likelihood is then evaluated via the map's Euclidean Distance Transform (EDT) $$p(z_t^q \mid \mathbf{x}_t^{[i]}, m) = \frac{1}{\sqrt{2\pi}\sigma_{\mathrm{obs}}} \exp\left( -\frac{d_{\mathrm{EDT}}(\hat{p}_q^{[i]})^2}{2\sigma_{\mathrm{obs}}^2} \right),$$ where $d_{\mathrm{EDT}}(\hat{p}_q^{[i]})$ returns the shortest distance from the endpoint to the nearest obstacle, and $\sigma_{\mathrm{obs}}$ is the sensor noise standard deviation.

Weight update: Assuming independent sensor beams, the joint likelihood for particle $i$ is the product of individual beam likelihoods $w_t^{[i]} \propto \prod_{q} p(z_t^q \mid \mathbf{x}_t^{[i]}, m)$. In practice, this is evaluated in log-space as $\sum_{q} \log p(z_t^q \mid \mathbf{x}_t^{[i]}, m)$ for numerical stability. The weights are subsequently normalized such that $\sum_{i=1}^{N} w_t^{[i]} = 1$.

Resampling and estimation: Resampling is triggered whenever the effective sample size $\zeta = \frac{1}{\sum_{i=1}^{N} \left(w_t^{[i]}\right)^2}$ falls below the threshold $\gamma N$, where $\gamma \in (0, 1)$ is a predefined hyperparameter. Systematic low-variance sampling is used to draw the new particle set. The estimated 2D position $\mathbf{\hat{x}}_t$ is calculated as the weighted particle mean $\mathbf{\hat{x}}_t = \sum_{i=1}^{N} w_t^{[i]} \begin{bmatrix} x_t^{[i]} \\ y_t^{[i]} \end{bmatrix}$.

\subsection{Augmented Filtered Monte Carlo Localization (AF-MCL) for Dynamic Indoor Environments}
\begin{algorithm}[!t]
\small
\caption{AF-MCL Update at Time Step $t$ (Dynamic Indoor Environments)}
\label{alg:amcl}
\begin{algorithmic}
\State \textbf{Input:} Map $m$, previous particle set $\{\mathbf{x}_{t-1}^{[i]}\}_{i=1}^N$, previous estimate $\hat{\mathbf{x}}_{t-1}$, motion input $\mathbf{u}_t$, measured yaw $\theta_t$, range measurements $\{z_t^q\}$, running likelihood averages $\bar{\ell}_{\mathrm{slow}}, \bar{\ell}_{\mathrm{fast}}$, decay rates $\alpha_{\mathrm{slow}} \ll \alpha_{\mathrm{fast}}$, injection cap $\nu$, beam threshold $\chi$
\State \textbf{Output:} Updated particle set $\{\mathbf{x}_t^{[i]}, w_t^{[i]}\}_{i=1}^N$, estimated position $\hat{\mathbf{x}}_t$, covariance $\hat{\Sigma}_t$, uncertainty $\sigma_t$, updated $\bar{\ell}_{\mathrm{slow}}, \bar{\ell}_{\mathrm{fast}}$

\State \textbf{Static-beam classification (at previous estimate):}
\For{each sensor $q$}
    \State $\hat{z}_t^q \leftarrow \mathrm{RayCast}\big(m,\hat{\mathbf{x}}_{t-1},\theta_t+\phi_q\big)$
\EndFor
\State $Q_t \leftarrow \left\{ q : \dfrac{z_{\mathrm{hit}}\,p_{\mathrm{hit}}(z_t^q\mid\hat{z}_t^q)}{p(z_t^q\mid\hat{z}_t^q)} > \chi \right\}$

\State \textbf{Motion update:}
\For{each particle $i = 1,\dots,N$}
    \State $\mathbf{x}_t^{[i]} \leftarrow \mathbf{x}_{t-1}^{[i]} + R(\theta_t)\mathbf{u}_t\Delta t + \epsilon_t^{[i]}, \quad \epsilon_t^{[i]} \sim \mathcal{N}(\mathbf{0},\Sigma_u\Delta t)$
\EndFor

\If{$Q_t = \emptyset$} \Comment{motion-only update}
    \State $w_t^{[i]} \leftarrow 1/N$; \ $\bar{\ell}_{\mathrm{slow}}, \bar{\ell}_{\mathrm{fast}}$ unchanged
\Else
    \State $\ell_t^{[i]} \leftarrow \exp\!\left(\dfrac{1}{|Q_t|}\sum_{q\in Q_t} \log p\big(z_t^q \mid \mathbf{x}_t^{[i]}, m\big)\right), \quad i=1,\dots,N$ \Comment{unnormalized}
    \State $w_t^{[i]} \leftarrow \ell_t^{[i]} \big/ \sum_{j=1}^N \ell_t^{[j]}$ \Comment{normalize}
    \State \textbf{Divergence detection (unnormalized likelihoods):}
    \State $\bar{\ell}_t \leftarrow \frac{1}{N}\sum_{i=1}^N \ell_t^{[i]}$
    \State $\bar{\ell}_{\mathrm{slow}} \leftarrow (1-\alpha_{\mathrm{slow}})\,\bar{\ell}_{\mathrm{slow}} + \alpha_{\mathrm{slow}}\,\bar{\ell}_t$
    \State $\bar{\ell}_{\mathrm{fast}} \leftarrow (1-\alpha_{\mathrm{fast}})\,\bar{\ell}_{\mathrm{fast}} + \alpha_{\mathrm{fast}}\,\bar{\ell}_t$
    \State $p_{\mathrm{random}} \leftarrow \min\!\left(\nu,\ \max\!\left(0,\ 1-\dfrac{\bar{\ell}_{\mathrm{fast}}}{\bar{\ell}_{\mathrm{slow}}}\right)\right)$
    \State \textbf{Augmented resampling:}
    \State $K_{\mathrm{rand}} \leftarrow \lfloor p_{\mathrm{random}} N \rfloor, \qquad K_{\mathrm{res}} \leftarrow N - K_{\mathrm{rand}}$
    \State $\mathcal{X}_{\mathrm{res}} \leftarrow \mathrm{SystematicResample}\big(\{\mathbf{x}_t^{[i]}\},\{w_t^{[i]}\},K_{\mathrm{res}}\big)$
    \State $\mathcal{X}_{\mathrm{rec}} \leftarrow \left\{\mathbf{x}_{\mathrm{rec}}^{[j]} \sim \mathcal{U}\left(m\right)  \ \middle|\ \mathbf{x}_{\mathrm{rec}}^{[j]} \text{ in free space of } m \right\}_{j=1}^{K_{\mathrm{rand}}}$
    \State $\{\mathbf{x}_t^{[i]}\}_{i=1}^N \leftarrow \mathcal{X}_{\mathrm{res}} \cup \mathcal{X}_{\mathrm{rec}}, \quad w_t^{[i]} \leftarrow 1/N$ \Comment{reset weights}
\EndIf

\State \textbf{Estimate and uncertainty (final particle set):}
\State $\hat{\mathbf{x}}_t \leftarrow \frac{1}{N}\sum_{i=1}^N \mathbf{p}_t^{[i]}$ \Comment{$\mathbf{p}_t^{[i]}=[x_t^{[i]},y_t^{[i]}]^\top$}
\State $\hat{\Sigma}_t \leftarrow \frac{1}{N}\sum_{i=1}^N (\mathbf{p}_t^{[i]}-\hat{\mathbf{x}}_t)(\mathbf{p}_t^{[i]}-\hat{\mathbf{x}}_t)^\top, \quad \sigma_t \leftarrow \sqrt{\mathrm{tr}(\hat{\Sigma}_t)}$
\State \Return $\{\mathbf{x}_t^{[i]}, w_t^{[i]}\}_{i=1}^N,\ \hat{\mathbf{x}}_t,\ \hat{\Sigma}_t,\ \sigma_t,\ \bar{\ell}_{\mathrm{slow}},\ \bar{\ell}_{\mathrm{fast}}$
\end{algorithmic}
\end{algorithm}

Algorithm~\ref{alg:amcl} extends the MCL algorithm by incorporating dynamic obstacle handling, a recovery mechanism, and an online estimate of localization uncertainty. The following subsections describe these additional stages.

Static-beam classification: To handle unknown dynamic obstacles, an inline beam-filtering mechanism is applied before weight updates. For each sensor $q$, a ray is cast through the map $m$ from the previous estimated pose $\mathbf{\hat{x}}_{t-1}$ along direction $(\theta_t + \phi_q)$ to determine the expected static range $\hat{z}_t^q$. Since the UAV's displacement between consecutive time steps is short (less than half of the UAV's diameter), using $\hat{\mathbf{x}}_{t-1}$ in place of per-particle ray casting, which is computationally infeasible under onboard hardware constraints, is an acceptable approximation. The received measurement $z_t^q$ is evaluated against a sensor mixture model $p(z_t^q \mid \hat{z}_t^q)= z_{\mathrm{hit}} p_{\mathrm{hit}}(z_t^q \mid \hat{z}_t^q) + z_{\mathrm{short}} p_{\mathrm{short}}(z_t^q \mid \hat{z}_t^q)$, where $z_{\mathrm{hit}}$ and $z_{\mathrm{short}}$ are mixing weights. The component densities are defined as
\begin{align*}
p_{\mathrm{hit}}(z_t^q \mid \hat{z}_t^q) &= \frac{1}{\sqrt{2\pi}\sigma_{\mathrm{obs}}} \exp\left(-\frac{(z_t^q - \hat{z}_t^q)^2}{2\sigma_{\mathrm{obs}}^2}\right),\\
p_{\mathrm{short}}(z_t^q \mid \hat{z}_t^q) &= 
\begin{cases} 
\kappa \, \lambda_{\mathrm{short}} e^{-\lambda_{\mathrm{short}} z_t^q} & \text{if } 0 \le z_t^q \le \hat{z}_t^q, \\
0 & \text{otherwise},
\end{cases}
\end{align*}
where $\lambda_{\mathrm{short}}$ is the exponential decay rate parameter modeling unexpected obstacle occurrences, and $\kappa = (1 - e^{-\lambda_{\mathrm{short}} \hat{z}_t^q})^{-1}$ is a normalization factor. Max-range ($z_{\max}$) and random noise ($z_{\mathrm{r}}$) components are retained with fixed weights so that all four mixing weights sum to 1; they contribute only to the denominator of the classification ratio and thus make the test slightly more conservative. Let $Q_t$ be the valid beam set. A sensor beam $q$ is valid and retained for particle weighting (added to the $Q_t$) if its posterior hit probability satisfies $$\frac{z_{\mathrm{hit}} p_{\mathrm{hit}}(z_t^q \mid \hat{z}_t^q)}{p(z_t^q \mid \hat{z}_t^q)} > \chi,$$ where $\chi \in (0, 1)$ is a decision threshold. Otherwise, the beam is classified as dynamic and excluded. Evaluating $\hat{z}_t^q$ solely at the previous pose estimate $\mathbf{\hat{x}}_{t-1}$ keeps the ray-casting overhead at $\mathcal{O}(1)$ per beam.

Divergence detection: To recover from global localization failures or kidnapping, an augmented recovery mechanism dynamically injects random particles when filter divergence is detected. For each particle, the unnormalized likelihood computed over the valid beams $Q_t$ is $\ell_t^{[i]} = \exp\left( \frac{1}{|Q_t|} \sum_{q \in Q_t} \log p(z_t^q \mid \mathbf{x}_t^{[i]}, m) \right),$ and the normalized importance weights are obtained as $w_t^{[i]} = \ell_t^{[i]} / \sum_{j=1}^{N}\ell_t^{[j]}$. Since normalized weights always satisfy $\frac{1}{N}\sum_i w_t^{[i]}=\frac{1}{N}$, the health of the filter is evaluated from the absolute average measurement likelihood across valid beams $Q_t$, computed before normalization as $\bar{\ell}_t = \frac{1}{N} \sum_{i=1}^{N} \ell_t^{[i]}.$
Two exponential moving averages of $\bar{\ell}_t$ are maintained
\begin{align*}
\bar{\ell}_{\mathrm{slow}} &\leftarrow (1 - \alpha_{\mathrm{slow}}) \bar{\ell}_{\mathrm{slow}} + \alpha_{\mathrm{slow}} \bar{\ell}_t, \\
\bar{\ell}_{\mathrm{fast}} &\leftarrow (1 - \alpha_{\mathrm{fast}}) \bar{\ell}_{\mathrm{fast}} + \alpha_{\mathrm{fast}} \bar{\ell}_t,
\end{align*}
where $\alpha_{\mathrm{slow}}$ and $\alpha_{\mathrm{fast}}$ are decay parameters satisfying $0 \le \alpha_{\mathrm{slow}} \ll \alpha_{\mathrm{fast}} \le 1$, and both averages are initialized with the first computed value of $\bar{\ell}_t$. When the filter performs poorly, $\bar{\ell}_{\mathrm{fast}}$ drops below $\bar{\ell}_{\mathrm{slow}}$, signaling divergence. The probability of injecting random particles, $p_{\mathrm{random}}$, is dynamically computed as $p_{\mathrm{random}} = \min \left( \nu, \, \max\left(0, \, 1 - \frac{\bar{\ell}_{\mathrm{fast}}}{\bar{\ell}_{\mathrm{slow}}}\right)\right)$, where $\nu \in (0, 1]$ caps the maximum proportion of random particles injected per iteration. The averages are updated only when $Q_t \neq \emptyset$.

Augmented resampling: When filter divergence is detected, recovery particles are sampled from a 2D uniform distribution covering the entire map $\mathbf{x}_{\mathrm{rec}}^{[i]} \sim \mathcal{U}\left(m\right)$, where 
$\mathcal{U}\left(m\right)$ denotes the uniform distribution over the map 
$m$, ensuring recovery particles are spread across the full extent of the environment, with drawn poses constrained to unoccupied space in the map $m$.

Given $K_{\mathrm{rand}} = \lfloor p_{\mathrm{random}} N \rfloor$ injected recovery particles and $K_{\mathrm{res}} = N - K_{\mathrm{rand}}$ systematic resampled particles, the updated particle set $\mathcal{X}_t$ is constructed as
\begin{align*}
\mathcal{X}_t = \underbrace{\left\{ \mathbf{x}_t^{[i]} \sim \mathrm{Resample}(w_t) \right\}_{i=1}^{K_{\mathrm{res}}}}_{\text{Importance Resampling}} \;\cup\; \underbrace{\left\{ \mathbf{x}_{\mathrm{rec}}^{[i]} \sim \mathcal{U}\left(m\right) \right\}_{i=1}^{K_{\mathrm{rand}}}}_{\text{Global Recovery}}.
\end{align*}

Online localization uncertainty: Besides the position estimate, AF-MCL outputs an online measure of how reliable that estimate is, computed at every time step from the particle set at negligible additional cost. The particle set approximates the posterior $\mathrm{bel}(\mathbf{x}_t)$, so its spread quantifies the remaining positional uncertainty. The estimate and its uncertainty are computed from the final particle set. The particles injected for recovery are therefore included, so the uncertainty increases when recovery particles are injected, which signals that the estimate is less trustworthy. The mean and covariance of particles $\mathbf{p}_t^{[i]} = \begin{bmatrix} x_t^{[i]} \\ y_t^{[i]} \end{bmatrix}, i=1,\dots,N$ are $\hat{\mathbf{x}}_t = \frac{1}{N}\sum_{i=1}^{N} \mathbf{p}_t^{[i]}$ and $\hat{\Sigma}_t = \frac{1}{N}\sum_{i=1}^{N}
\left(\mathbf{p}_t^{[i]} - \hat{\mathbf{x}}_t\right)\left(\mathbf{p}_t^{[i]} - \hat{\mathbf{x}}_t\right)^\top$. The covariance matrix $\hat{\Sigma}_t \in \mathbb{R}^{2\times2}$ has the form $\hat{\Sigma}_t = \begin{bmatrix} \hat{\sigma}_{x,t}^2 & \hat{\sigma}_{xy,t} \\ \hat{\sigma}_{xy,t} & \hat{\sigma}_{y,t}^2 \end{bmatrix},$ where $\hat{\sigma}_{x,t}^2$ and $\hat{\sigma}_{y,t}^2$ are the variances of the particle positions along the $x$ and $y$ axes, and $\hat{\sigma}_{xy,t}$ is their covariance. It is summarized by the scalar uncertainty $\sigma_t = \sqrt{\mathrm{tr}(\hat{\Sigma}_t)} = \sqrt{\hat{\sigma}_{x,t}^2+\hat{\sigma}_{y,t}^2},$ where $\mathrm{tr}(\cdot)$ denotes the matrix trace, i.e., the sum of the diagonal entries. Thus, $\sigma_t$ is the root-mean-square spread of the particles around the estimate. A small $\sigma_t$ indicates a concentrated belief, whereas a large or growing $\sigma_t$ indicates insufficient measurement correction (e.g., many rejected beams) or filter divergence. Since particle filters tend to be overconfident, $\sigma_t$ is interpreted as the uncertainty reported by the filter rather than a guaranteed error bound. $\sigma_t$ describes how uncertain the filter currently is. The pair $(\hat{\mathbf{x}}_t, \hat{\Sigma}_t)$ can be passed to downstream modules (e.g., a controller or planner).

\textit{Empty beam set ($Q_t=\emptyset$):} If all beams are rejected, the geometric-mean likelihood is undefined and the update becomes motion-only. Particles are propagated, weights are reset to $1/N$, and resampling is skipped. Since no valid measurement likelihood is available, the running likelihood averages remain unchanged and $p_{\mathrm{random}}=0$, so no recovery particles are injected. Consequently, persistent beam rejection caused by a large localization error may delay recovery, since the divergence detector receives no valid likelihood update. This case occurred rarely in our experiments. A possible extension is to monitor consecutive empty-beam updates and, after a threshold, temporarily include all in-range beams in weighting and divergence detection to enable recovery.

\section{Theoretical Analysis}
\label{sec:theoretical_analysis}

AF-MCL classifies beams only at the previous position estimate, avoiding
particle-wise ray casting but potentially reducing the geometric
information available after inconsistent beams are rejected. We
characterize these effects through computational complexity and the rank
and conditioning of the retained measurement geometry. The latter
analysis is local and is not a convergence analysis of the particle
filter. Let \(M\) and \(N\) denote the numbers of sensors and particles,
\(Q_t\subseteq\{1,\ldots,M\}\) the retained beam set, and
\(C_{\mathrm{ray}}\) the cost of one ray cast. The cost of an Euclidean distance transform lookup and
single-beam likelihood evaluation are assumed to have bounded cost
independent of \(N\).

\subsection{Computational Complexity}
\label{subsec:complexity}

AF-MCL classifies each beam once at
\(\widehat{\mathbf{p}}_{t-1}\), rather than independently at every
particle.

\begin{proposition}[Computational complexity of AF-MCL]
\label{prop:complexity}
One AF-MCL update has complexity
\[
\mathcal{O}\!\left(MC_{\mathrm{ray}}+N|Q_t|+N\right)
\subseteq
\mathcal{O}\!\left(MC_{\mathrm{ray}}+NM+N\right).
\]
In contrast, particle-wise beam classification requires
\(\mathcal{O}(NMC_{\mathrm{ray}}+NM+N)\). Thus, the proposed strategy
reduces the ray-casting cost from
\(\mathcal{O}(NMC_{\mathrm{ray}})\) to
\(\mathcal{O}(MC_{\mathrm{ray}})\).
\end{proposition}

\begin{proof}
Classification performs \(M\) ray casts and costs
\(\mathcal{O}(MC_{\mathrm{ray}})\). Particle propagation and the
remaining particle-wise operations cost \(\mathcal{O}(N)\), while
weighting costs \(\mathcal{O}(N|Q_t|)\). Combining these terms gives
the stated complexity, with \(|Q_t|\leq M\). Particle-wise
classification instead requires \(NM\) ray casts, yielding the stated
\(\mathcal{O}(NMC_{\mathrm{ray}}+NM+N)\) cost.
\end{proof}

Hence, the additional ray-casting cost introduced by beam
classification is independent of \(N\). With \(M=4\), AF-MCL performs
four classification ray casts per update.

\subsection{First-Order Geometry of the Retained Beams}
\label{subsec:first_order_geometry}

Filtering removes measurements inconsistent with the static map, but
also changes the geometric information available for localization. Let
\(\mathbf{p}_t=[x_t\;y_t]^\top\). Since \(\theta_t\) is measured, the
map-predicted range of sensor \(q\) can be written locally as
\(\widehat z_t^q=h_q(\mathbf{p}_t;\theta_t,m)\). For the retained set,
define
\[
\mathbf{h}_{Q_t}(\mathbf{p})
=\operatorname{col}_{q\in Q_t}h_q(\mathbf{p};\theta_t,m),
\quad
H_{Q_t}(\mathbf{p})
=\frac{\partial\mathbf{h}_{Q_t}}{\partial\mathbf{p}}
\in\mathbb{R}^{|Q_t|\times2}.
\]
Because ray-cast ranges are piecewise smooth, the following analysis
applies locally where the map boundary intersected by each retained
beam remains unchanged and the corresponding \(h_q\) are continuously
differentiable.

We call the position \emph{first-order geometrically identifiable} at
\(\mathbf{p}^{\star}\) if
\(H_{Q_t}(\mathbf{p}^{\star})\mathbf{v}=\mathbf{0}\) implies
\(\mathbf{v}=\mathbf{0}\).

\begin{proposition}[First-order geometric identifiability]
\label{prop:observability}
The position is first-order geometrically identifiable at
\(\mathbf{p}^{\star}\) if and only if
\[
\operatorname{rank}H_{Q_t}(\mathbf{p}^{\star})=2,
\]
equivalently, if two retained beams \(i,j\in Q_t\) have non-collinear
gradients,
\(\nabla h_i(\mathbf{p}^{\star})\not\parallel
\nabla h_j(\mathbf{p}^{\star})\).
\end{proposition}

\begin{proof}
First-order identifiability is equivalent to
\(\operatorname{null}H_{Q_t}=\{\mathbf{0}\}\), which, for a
two-dimensional position, holds if and only if
\(\operatorname{rank}H_{Q_t}=2\). Since the rows of \(H_{Q_t}\) are
the transposed range gradients, this is equivalent to at least two
non-collinear retained gradients.
\end{proof}

Thus, the number of retained beams alone does not determine the
available geometric information. For example, retaining only the front
and back beams can yield anti-parallel gradients and
\(\operatorname{rank}H_{Q_t}=1\), leaving one position direction
unresolved to first order. In particular, \(|Q_t|<2\) cannot provide
rank two, while \(|Q_t|\geq2\) may still be rank deficient if the
retained gradients are collinear.

\subsection{Conditioning and Local Measurement Sensitivity}
\label{subsec:conditioning}

Full rank establishes first-order identifiability but not its
conditioning. The smallest singular value
\(\sigma_{\min}(H_{Q_t})\) quantifies the weakest geometric constraint,
with values approaching zero indicating increasing sensitivity to
measurement perturbations.

Consider
\(\mathbf{z}_{Q_t}=\mathbf{h}_{Q_t}(\mathbf{p}^{\star})
+\boldsymbol{\eta}_t\), where \(\boldsymbol{\eta}_t\) includes range
noise, map error, and residual measurement inconsistency. For a local
position \(\widehat{\mathbf{p}}\), define
\(\widetilde{\mathbf{p}}=\widehat{\mathbf{p}}-\mathbf{p}^{\star}\) and
\(\mathbf{r}=\mathbf{z}_{Q_t}
-\mathbf{h}_{Q_t}(\widehat{\mathbf{p}})\).

\begin{proposition}[Local sensitivity of the retained measurements]
\label{prop:error_bound}
Under the assumptions of Proposition~\ref{prop:observability}, if
\(\sigma_{\min}(H_{Q_t}(\mathbf{p}^{\star}))>0\), then, as
\(\|\widetilde{\mathbf{p}}\|_2\rightarrow0\),
\begin{equation}
\|\widetilde{\mathbf{p}}\|_2
\leq
\frac{\|\boldsymbol{\eta}_t\|_2+\|\mathbf{r}\|_2}
{\sigma_{\min}(H_{Q_t}(\mathbf{p}^{\star}))}
+
\mathcal{O}\!\left(
\frac{\|\widetilde{\mathbf{p}}\|_2^2}
{\sigma_{\min}(H_{Q_t}(\mathbf{p}^{\star}))}
\right).
\label{eq:local_sensitivity_bound}
\end{equation}
For an exact local fit (\(\mathbf{r}=\mathbf{0}\)), the numerator of
the leading term reduces to \(\|\boldsymbol{\eta}_t\|_2\).
\end{proposition}

\begin{proof}
A Taylor expansion about \(\mathbf{p}^{\star}\) gives
\[
\mathbf{h}_{Q_t}(\widehat{\mathbf{p}})
=
\mathbf{h}_{Q_t}(\mathbf{p}^{\star})
+
H_{Q_t}(\mathbf{p}^{\star})\widetilde{\mathbf{p}}
+
\mathcal{O}(\|\widetilde{\mathbf{p}}\|_2^2),
\]
and therefore
\[
H_{Q_t}(\mathbf{p}^{\star})\widetilde{\mathbf{p}}
=
\boldsymbol{\eta}_t-\mathbf{r}
+\mathcal{O}(\|\widetilde{\mathbf{p}}\|_2^2).
\]
Since \(H_{Q_t}\) has full column rank,
\(\|H_{Q_t}^{\dagger}\|_2=1/\sigma_{\min}(H_{Q_t})\).
Left-multiplication by \(H_{Q_t}^{\dagger}\) and the triangle
inequality yield \eqref{eq:local_sensitivity_bound}.
\end{proof}

Proposition~\ref{prop:error_bound} shows that larger
\(\sigma_{\min}(H_{Q_t})\) corresponds to stronger geometric
constraints, whereas values approaching zero imply greater sensitivity
to measurement perturbations; at zero, the retained geometry is rank
deficient. This characterizes the instantaneous retained measurement
geometry, not the AF-MCL estimation error, which also depends on prior
and motion information.


\section{Results and Discussion}
\label{result_and_discussion}
The MCL, A-MCL (augmentation only), F-MCL (beam filtering only), R-MCL (Reliable MCL~\cite{akai2023reliable}), and the full AF-MCL algorithms are evaluated in simulation, and AF-MCL is further validated in real-world experiments. The UAV follows its nonlinear dynamics under low-level PID control, with constant altitude, forward velocity, zero lateral velocity, and angular velocity references.
\subsection{Simulation Setup}

Simulations were run using a Crazyflie~2.1 model in Webots 2024a\footnote{\url{https://cyberbotics.com}} on Ubuntu 24.04 LTS (Intel Core i7, 16~GB RAM). The UAV flew at a constant altitude of 0.3~m with linear and angular velocities of 0.3~m/s and 0.4~rad/s, respectively. The filter was updated at 10~Hz ($\Delta t=0.1$~s) on an occupancy grid with 2~cm resolution, and the EDT was truncated at the sensor range limit $r_{\max}=10$~m. The four range sensors were mounted at the UAV center with offsets $\phi_q=0,\ \pi/2,\ \pi,\ -\pi/2$~rad (front, left, back, right), and the heading $\theta_t$ was taken noise-free from the simulator. Range measurements were corrupted with additive Gaussian noise with a standard deviation of 2~cm to emulate realistic sensor uncertainty; missing, non-finite, non-positive, or out-of-range ($\ge r_{\max}$) readings were discarded before weighting. Particles were initialized from $\mathcal{N}(\bar{\mathbf{x}}_0,\Sigma_0)$ around the known start position with uniform weights. The simulation environment is shown in Fig.~\ref{fig:webots}; parameters, determined empirically by manual tuning, are listed in Table~\ref{tab:sim_params}. $\Sigma_0$ and $\Sigma_u$ are isotropic and specified by their standard deviations (in m and m/$\sqrt{\mathrm{s}}$, respectively). The slow and fast decay rates satisfy $\alpha_{\mathrm{slow}}\ll\alpha_{\mathrm{fast}}$, $\sigma_{\mathrm{obs}}$ is set above the sensor noise to tolerate map discretization and pose error, and $\nu$ caps the injection so that most particles are always importance-resampled.

\begin{table}
    \centering
    \caption{Simulation parameters.}
    \begin{tabularx}{\linewidth}{
       >{\hsize=0.05\hsize}X
       >{\hsize=0.08\hsize}X
       >{\hsize=0.05\hsize}X
       >{\hsize=0.1\hsize}X
       >{\hsize=0.1\hsize}X
       >{\hsize=0.05\hsize}X
       >{\hsize=0.05\hsize}X
       >{\hsize=0.1\hsize}X
       >{\hsize=0.1\hsize}X
       >{\hsize=0.1\hsize}X
       >{\hsize=0.1\hsize}X
       >{\hsize=0.1\hsize}X
       >{\hsize=0.05\hsize}X
       >{\hsize=0.05\hsize}X
      }
        \toprule
        $N$ & $z_{\max}$ & $z_{\mathrm{r}}$ & $z_{\mathrm{short}}$ & $z_{\mathrm{hit}}$ & $\Sigma_0$ & $\gamma$ & $\lambda_{\mathrm{short}}$ & $\Sigma_u$ & $\sigma_{\mathrm{obs}}$ & $\alpha_{\text{slow}}$ & $\alpha_{\text{fast}}$ & $\chi$ & $\nu$ \\
        \midrule
        100 & 0.1 & 0.1 & 0.1 & 0.7 & 0.2 & 0.5 & 4.0 & 0.03 & 0.15 & 0.001 & 0.9 & 0.4 & 0.3 \\
        \bottomrule
    \end{tabularx}
    \label{tab:sim_params}
\end{table}
\begin{figure}
  \centering
  \includegraphics[width=0.5\linewidth]{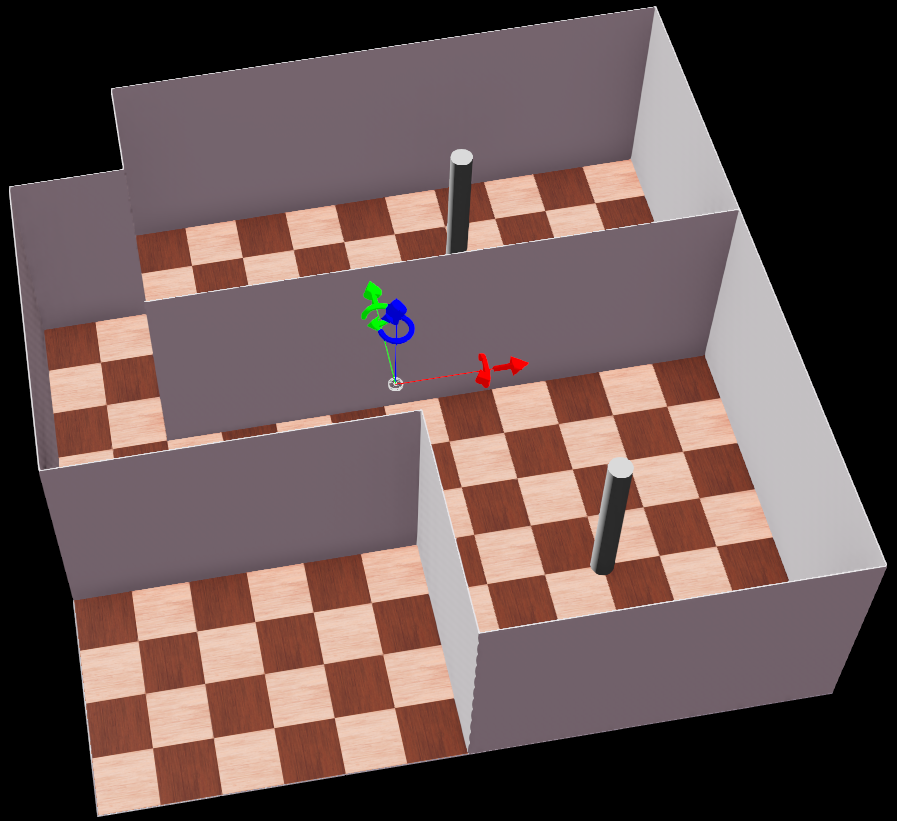}
  \caption{Webots simulation environment showing the Crazyflie UAV, static walls, and the dynamic obstacles used to test localization robustness.}
  \label{fig:webots}
\end{figure}

\subsection{Ablation and Comparative Study}
To evaluate each component's contribution, standard MCL, A-MCL (augmentation only), F-MCL (filtering only), the proposed AF-MCL, and R-MCL (Reliable MCL~\cite{akai2023reliable}) were evaluated using a 4-range LiDAR setup across square and S-shaped trajectories in $6\times6$\,m and $10\times10$\,m environments with dynamic obstacles (Fig.~\ref{fig:ablation_square_s-shape}, Table~\ref{tab:performance_results}).

\begin{figure}[t]
  \centering
  \begin{subfigure}{0.24\linewidth}
    \centering
    \includegraphics[width=\linewidth]{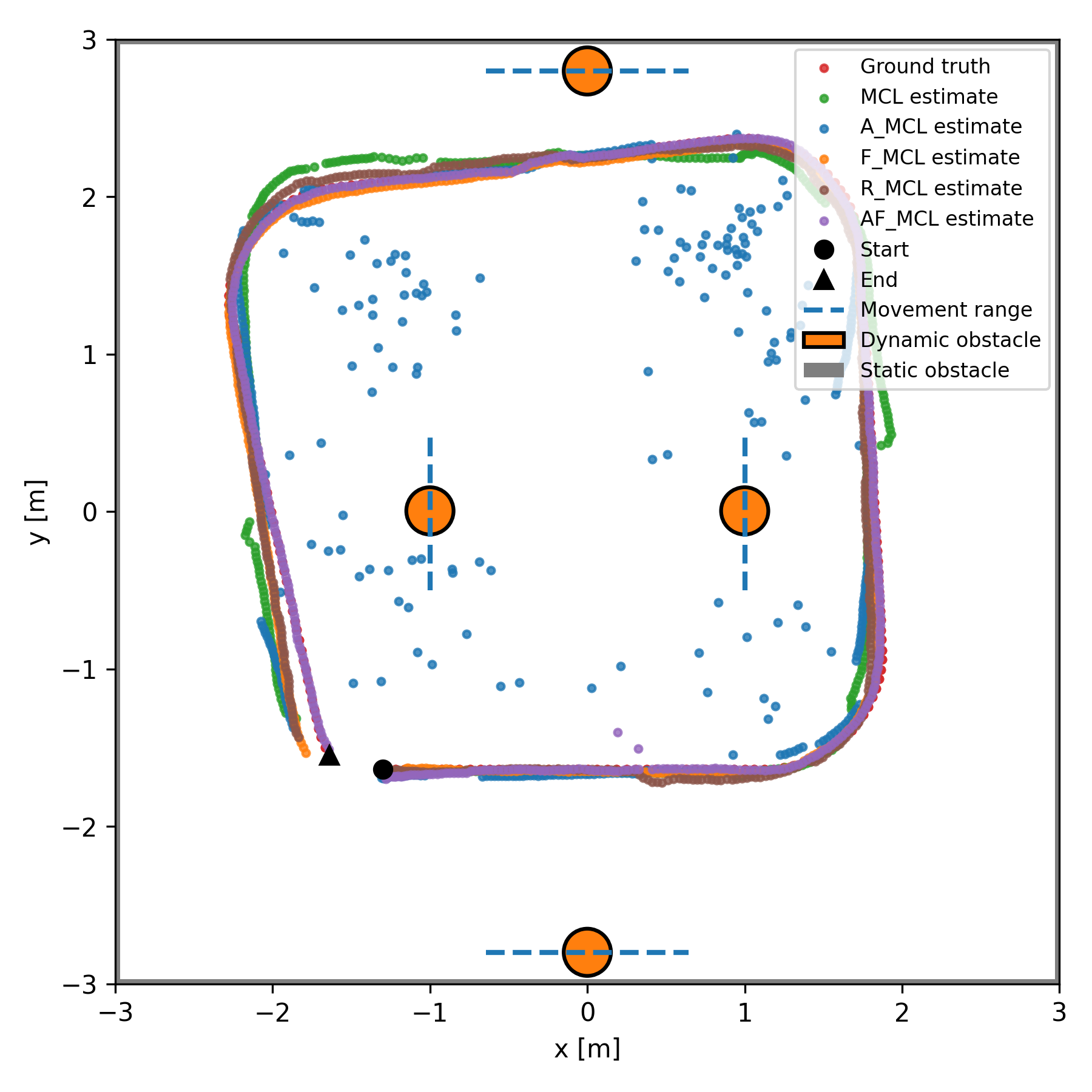}
    \caption{}
  \end{subfigure}
  \begin{subfigure}{0.24\linewidth}
    \centering
    \includegraphics[width=\linewidth]{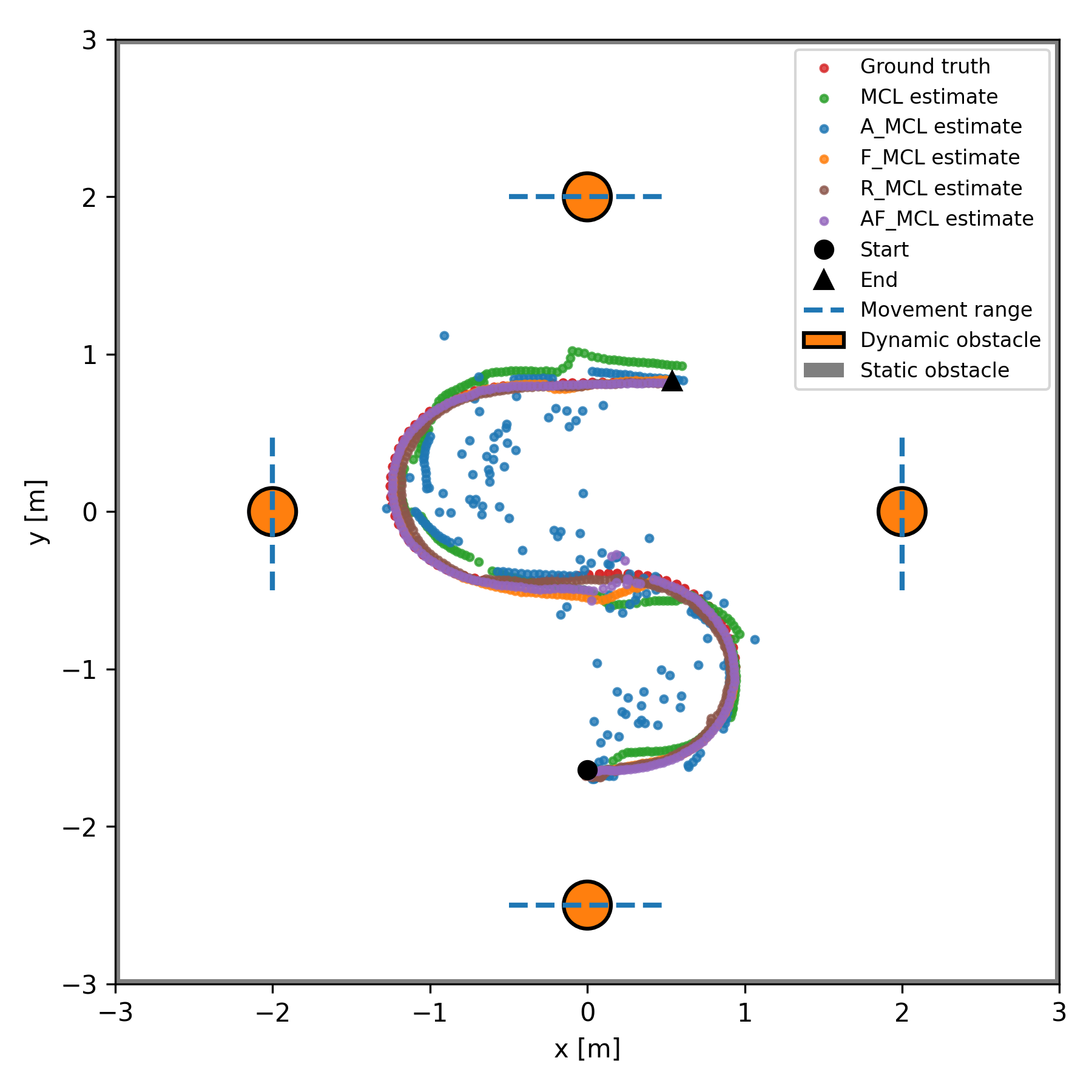}
    \caption{}
  \end{subfigure}
  \begin{subfigure}{0.24\linewidth}
    \centering
    \includegraphics[width=\linewidth]{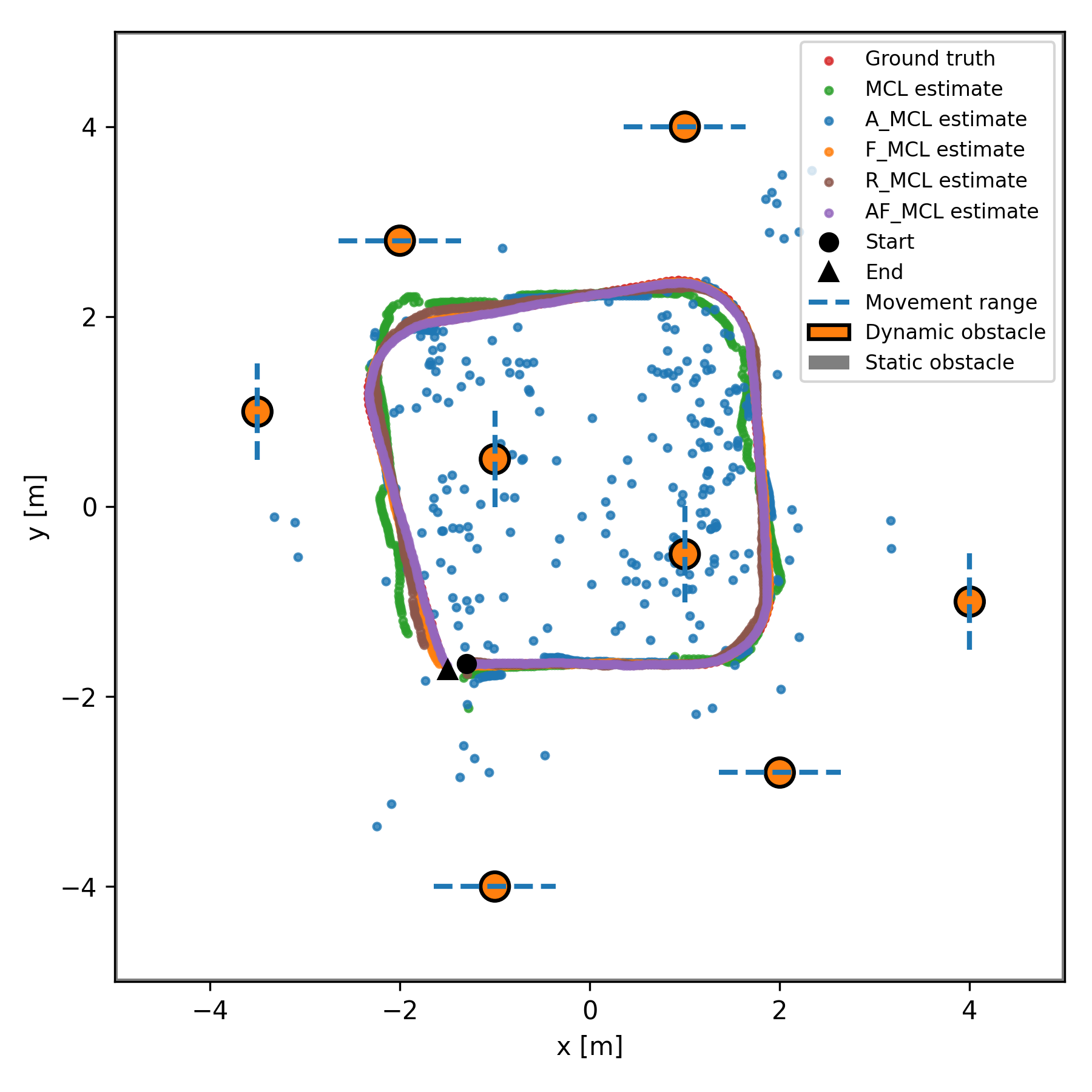}
    \caption{}
  \end{subfigure}
  \begin{subfigure}{0.24\linewidth}
    \centering
    \includegraphics[width=\linewidth]{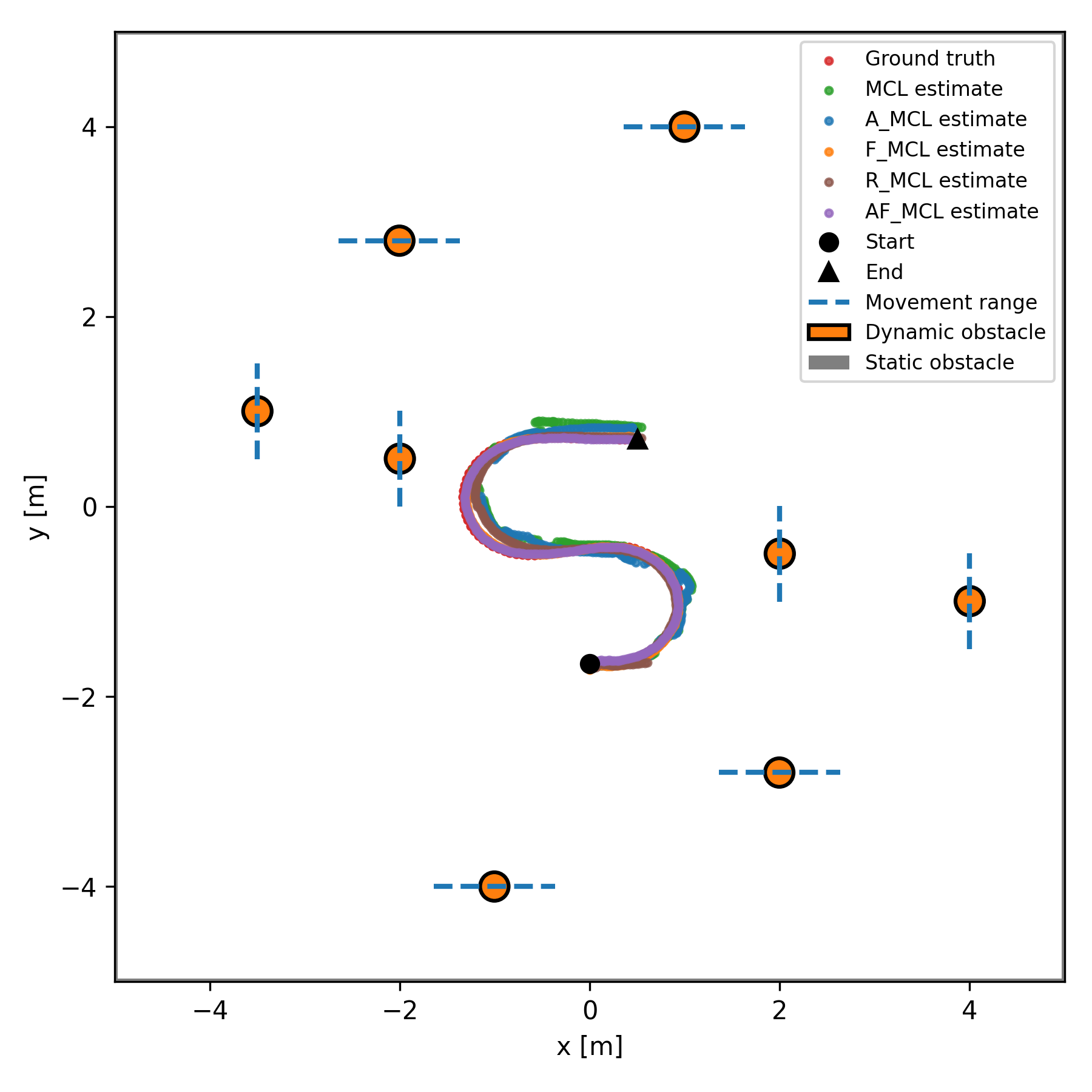}
    \caption{}
  \end{subfigure}
  \begin{subfigure}{0.24\linewidth}
    \centering
    \includegraphics[width=\linewidth]{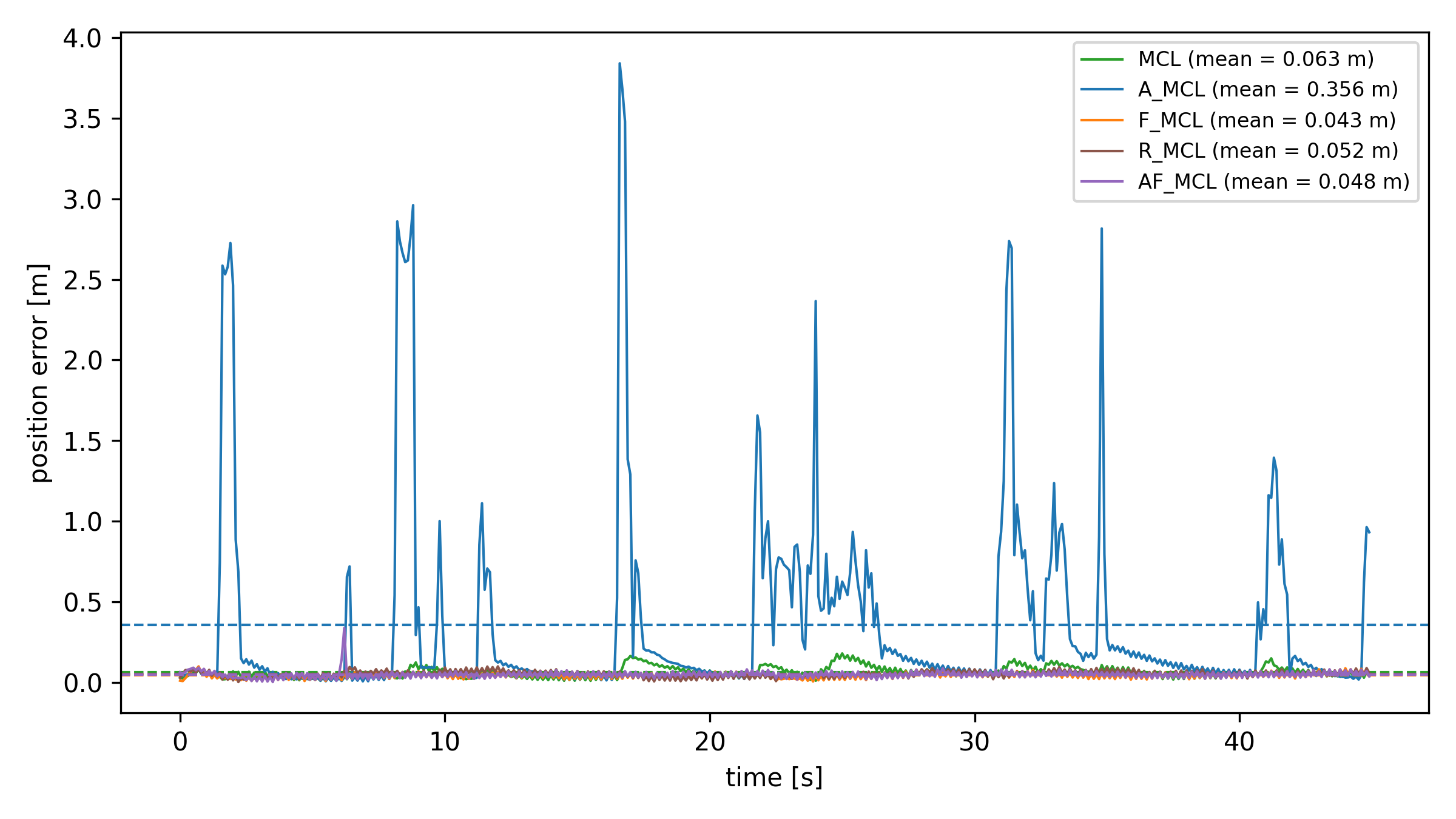}
    \caption{}
  \end{subfigure}
  \begin{subfigure}{0.24\linewidth}
    \centering
    \includegraphics[width=\linewidth]{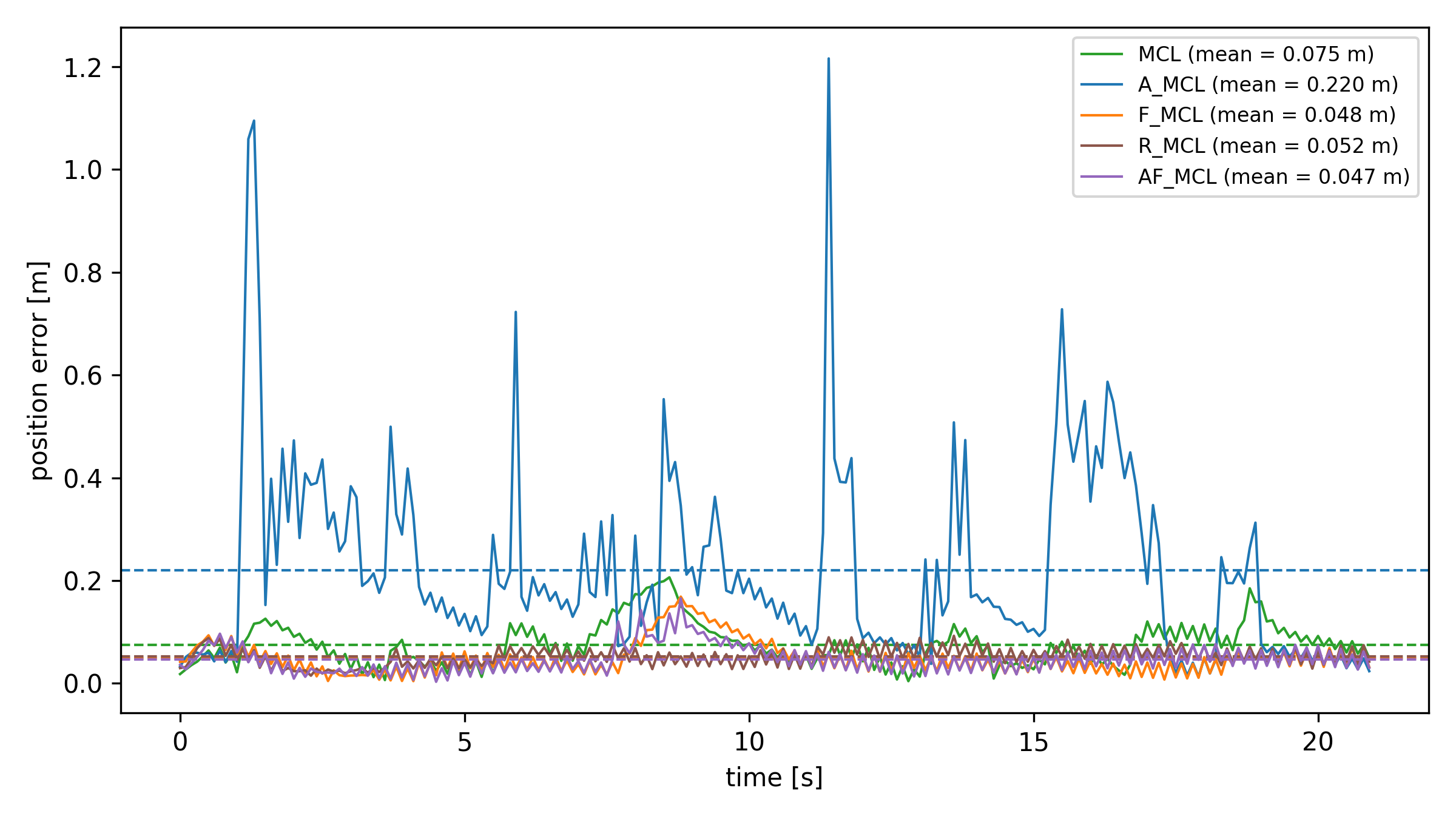}
    \caption{}
  \end{subfigure}
  \begin{subfigure}{0.24\linewidth}
    \centering
    \includegraphics[width=\linewidth]{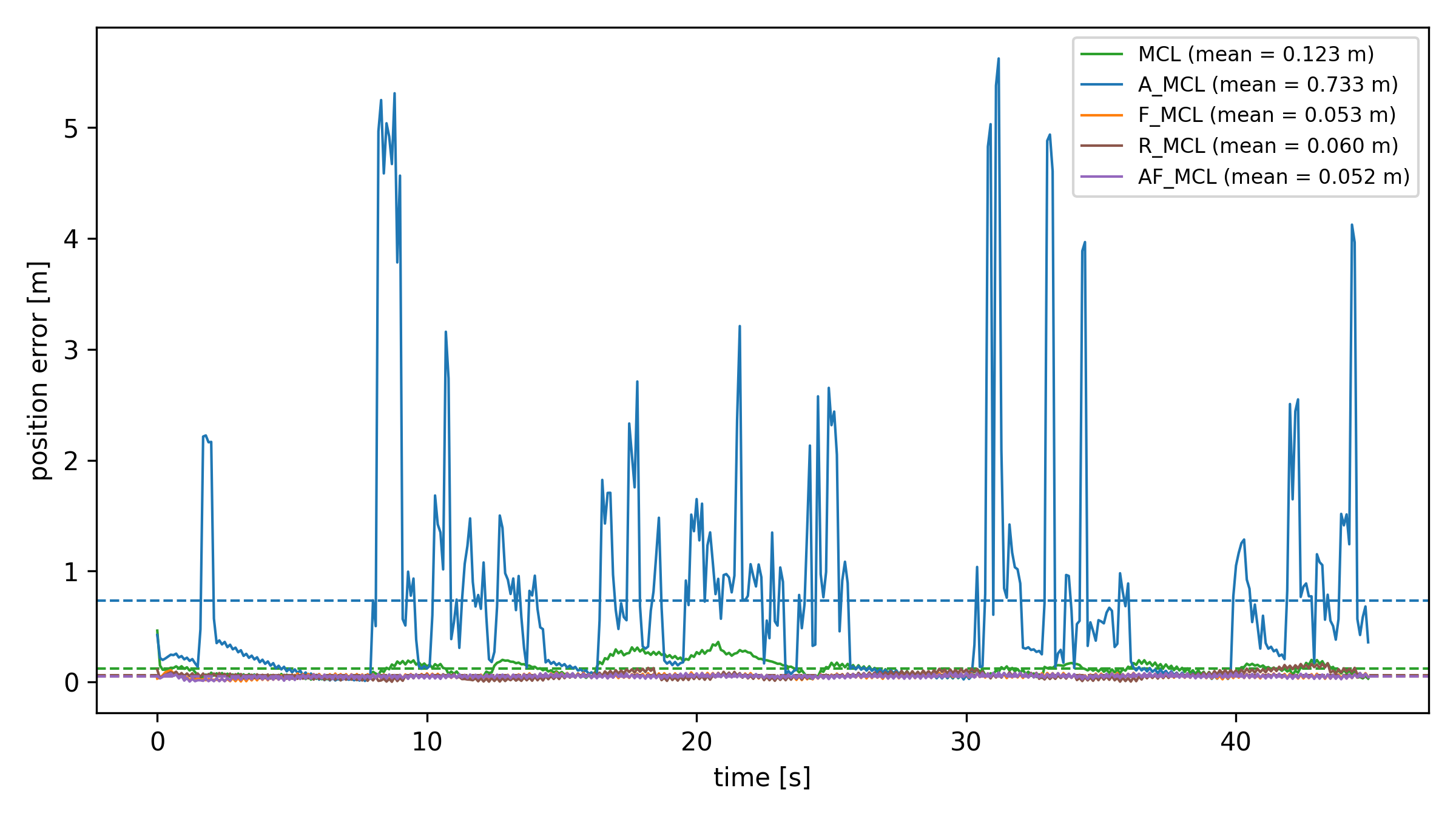}
    \caption{}
  \end{subfigure}
  \begin{subfigure}{0.24\linewidth}
    \centering
    \includegraphics[width=\linewidth]{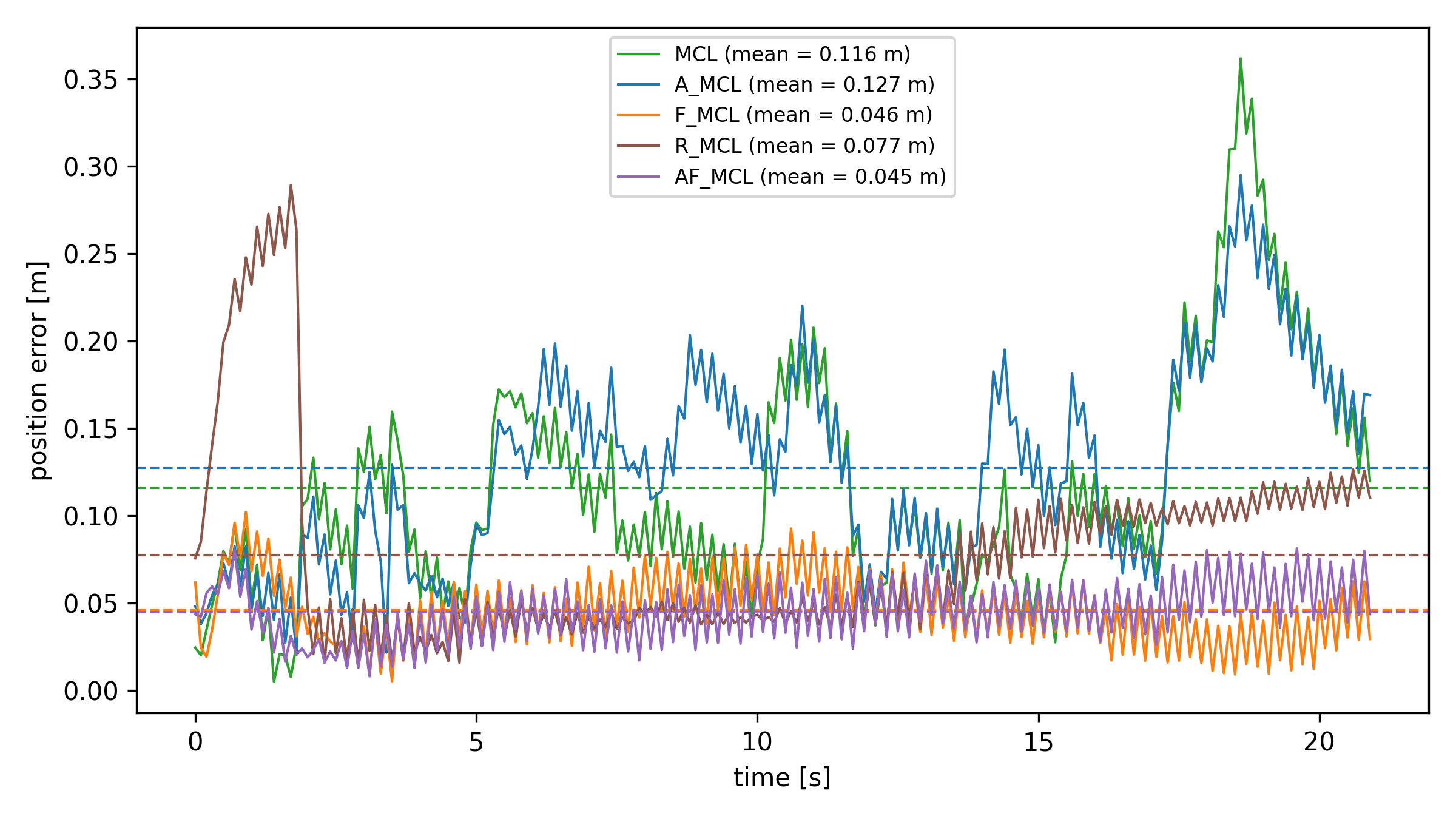}
    \caption{}
  \end{subfigure}
  \caption{Ablation results under dynamic obstacles: (a--d) estimated vs. ground-truth paths for square and S-shaped trajectories in $6\times6$\,m (a, b) and $10\times10$\,m (c, d) environments; (e--h) corresponding position errors over time.}
  \label{fig:ablation_square_s-shape}
\end{figure}

\begin{table}[t]
\centering
\caption{Position Error and Runtime Performance across Environments, Trajectories, and Methods}
\label{tab:performance_results}
\begin{tabularx}{\linewidth}{
   >{\hsize=0.12\hsize}X
  >{\hsize=0.12\hsize}X
  >{\hsize=0.21\hsize}X
  >{\hsize=0.27\hsize}X
  >{\hsize=0.32\hsize}X
  >{\hsize=0.19\hsize}X
  }
\toprule
\textbf{Env.} & \textbf{Traj.} & \textbf{Method} & \textbf{Mean Pos.} & \textbf{Mean} & \textbf{Mean $\sigma$} \\
 \text{[m]} &  &  & \textbf{Error} [m] & \textbf{Runtime} [ms] & [m] \\
\midrule
\multirow{10}{*}{$6\times6$} & \multirow{5}{*}{S-shape} 
    & MCL    & 0.075~$\pm$~0.041 & \textbf{0.68~$\pm$~0.67} & 0.048 \\
 &  & A-MCL  & 0.220~$\pm$~0.187 & 1.60~$\pm$~1.05 & 0.565\\
 &  & F-MCL  & 0.048~$\pm$~0.032 & 2.73~$\pm$~1.35 & 0.049 \\
 &  & R-MCL  & 0.052~$\pm$~0.018 & 2.10~$\pm$~1.15 & 0.049 \\
 &  & AF-MCL & \textbf{0.047~$\pm$~0.026} & 3.15~$\pm$~1.26 & 0.051 \\
\cmidrule{2-6}
 & \multirow{5}{*}{Square}  
    & MCL    & 0.063~$\pm$~0.034 & \textbf{1.04~$\pm$~0.41} & 0.049 \\
 &  & A-MCL  & 0.356~$\pm$~0.617 & 1.74~$\pm$~0.70 & 0.541 \\
 &  & F-MCL  & \textbf{0.043~$\pm$~0.019} & 3.12~$\pm$~0.92 & 0.048 \\
 &  & R-MCL  & 0.052~$\pm$~0.021 & 2.29~$\pm$~1.01 & 0.050 \\
 &  & AF-MCL & 0.048~$\pm$~0.023 & 3.48~$\pm$~0.88 & 0.047 \\
\midrule
\multirow{10}{*}{10$\times$10} & \multirow{5}{*}{S-shape} 
    & MCL    & 0.116~$\pm$~0.066 & \textbf{0.94~$\pm$~1.41} & 0.038 \\
 &  & A-MCL  & 0.127~$\pm$~0.057 & 1.29~$\pm$~2.61 & 0.038 \\
 &  & F-MCL  & 0.046~$\pm$~0.021 & 4.51~$\pm$~2.80 & 0.049 \\
 &  & R-MCL  & 0.077~$\pm$~0.056 & 1.90~$\pm$~1.97 & 0.149 \\
 &  & AF-MCL & \textbf{0.045~$\pm$~0.018} & 5.03~$\pm$~2.57 & 0.053\\
\cmidrule{2-6}
 & \multirow{5}{*}{Square}  
    & MCL    & 0.123~$\pm$~0.071 & \textbf{1.20~$\pm$~1.76} & 0.039 \\
 &  & A-MCL  & 0.733~$\pm$~1.048 & 1.83~$\pm$~1.71 &  1.314\\
 &  & F-MCL  & 0.053~$\pm$~0.019 & 4.81~$\pm$~0.95 & 0.049 \\
 &  & R-MCL  & 0.060~$\pm$~0.033 & 2.41~$\pm$~1.66 & 0.060 \\
 &  & AF-MCL & \textbf{0.052~$\pm$~0.019} & 4.86~$\pm$~1.71 & 0.050\\
\bottomrule
\end{tabularx}
\end{table}

Beam filtering drives accuracy under dynamic obstacles: F-MCL and AF-MCL achieve the lowest errors in both map sizes. Augmentation adds no measurable accuracy in this regime (AF-MCL and F-MCL differ by $\le$0.005\,m). Standalone augmentation (A-MCL) is harmful: readings from dynamic obstacles lower the likelihoods and trigger false divergence recovery, which scatters the particles. Filtering removes these readings, so injection is triggered mainly by genuine divergence. Augmentation therefore serves as a recovery mechanism for kidnapping or genuine filter failure, not as a source of accuracy during normal operation. The online uncertainty $\sigma$ is a reliable indicator only when beams are filtered. For F-MCL, AF-MCL, and R-MCL (except the $10\times10$\,m S-shape case), $\sigma$ is of the same magnitude as the mean error, so it provides a usable estimate of the localization uncertainty without ground truth. For standard MCL, $\sigma$ stays at about 0.04--0.05\,m while the error reaches 0.12\,m (about $3\times$ underestimated), so the particle cloud is overconfident when dynamic obstacles corrupt the likelihood. For A-MCL, false injections inflate $\sigma$ to 0.5--1.3\,m in three of four cases, which flags the unreliable estimate, but in the $10\times10$\,m S-shape case $\sigma$ stays low despite a large error. While R-MCL is more accurate than MCL and faster than AF-MCL, it requires extensive offline map preprocessing and memory-intensive lookup tables, making AF-MCL more suitable for onboard execution on resource-constrained micro-UAVs.

\subsection{Kidnapping Recovery} \label{subsec:kidnap}
A kidnapping scenario tests how quickly each filter recovers from a sudden, unmodeled displacement. A simulated office environment with two rooms and walls of varying widths is shown in Fig.~\ref{fig:simulation_comparison_kidnapping}. The UAV starts in one room; after 2~s, it is instantaneously "kidnapped" to a different location while the filter's belief still reflects the original position. The kidnapping lasts from 2~s to 12~s, once it ends, each algorithm must re-localize. Recovery time is defined as the time needed to re-converge to a low-error estimate after the kidnapping begins.

As shown in Fig.~\ref{fig:simulation_comparison_kidnapping}a, AF-MCL recovers in 2.9~s, compared to 16.9~s for MCL (nearly a 6x improvement). Fig.~\ref{fig:simulation_comparison_kidnapping}b shows the corresponding error over time: the mean estimation error during the trajectory is 0.191~m for AF-MCL versus 0.701~m for MCL. This gap is attributable to AF-MCL's augmentation strategy (injecting random particles when a divergence is detected), which lets it recover from sudden relocalization far faster than standard MCL.

Figs.~\ref{fig:simulation_comparison_kidnapping}c and~\ref{fig:simulation_comparison_kidnapping}d show the internal signals that drive this augmentation. During the recovery times, $p_{\mathrm{random}}$ rises above 0, meaning the filter starts injecting random particles into the belief instead of relying only on resampled ones. This is triggered by the drop in $\bar{\ell}_{\mathrm{fast}}$ below $\bar{\ell}_{\mathrm{slow}}$ over the same times: a sudden fall in short-term average weight relative to the long-term average signals that the current particle set no longer explains the incoming measurements, i.e., the filter has lost track of the true state. Together, these two signals confirm that the augmentation strategy activates exactly when needed and is the mechanism responsible for AF-MCL's fast recovery.
\begin{figure}[t]
  \centering
  \begin{subfigure}[b]{0.48\linewidth}
    \centering
    \includegraphics[width=\linewidth]{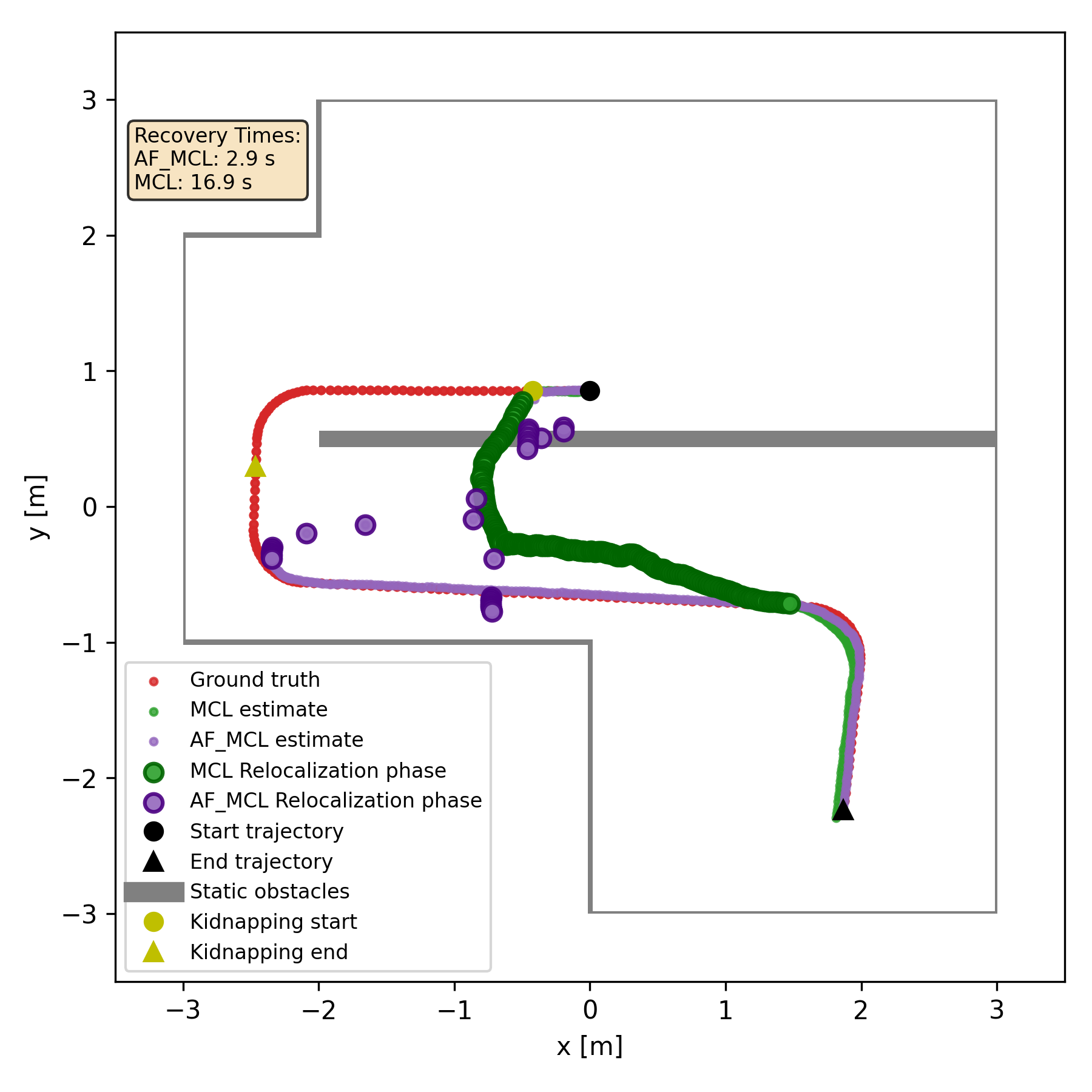}
    \caption{}
    \label{fig:kidnapping_traj}
  \end{subfigure}\hfill
  \begin{subfigure}[b]{0.48\linewidth}
    \centering
    \begin{subfigure}{\linewidth}
      \centering
      \includegraphics[width=\linewidth]{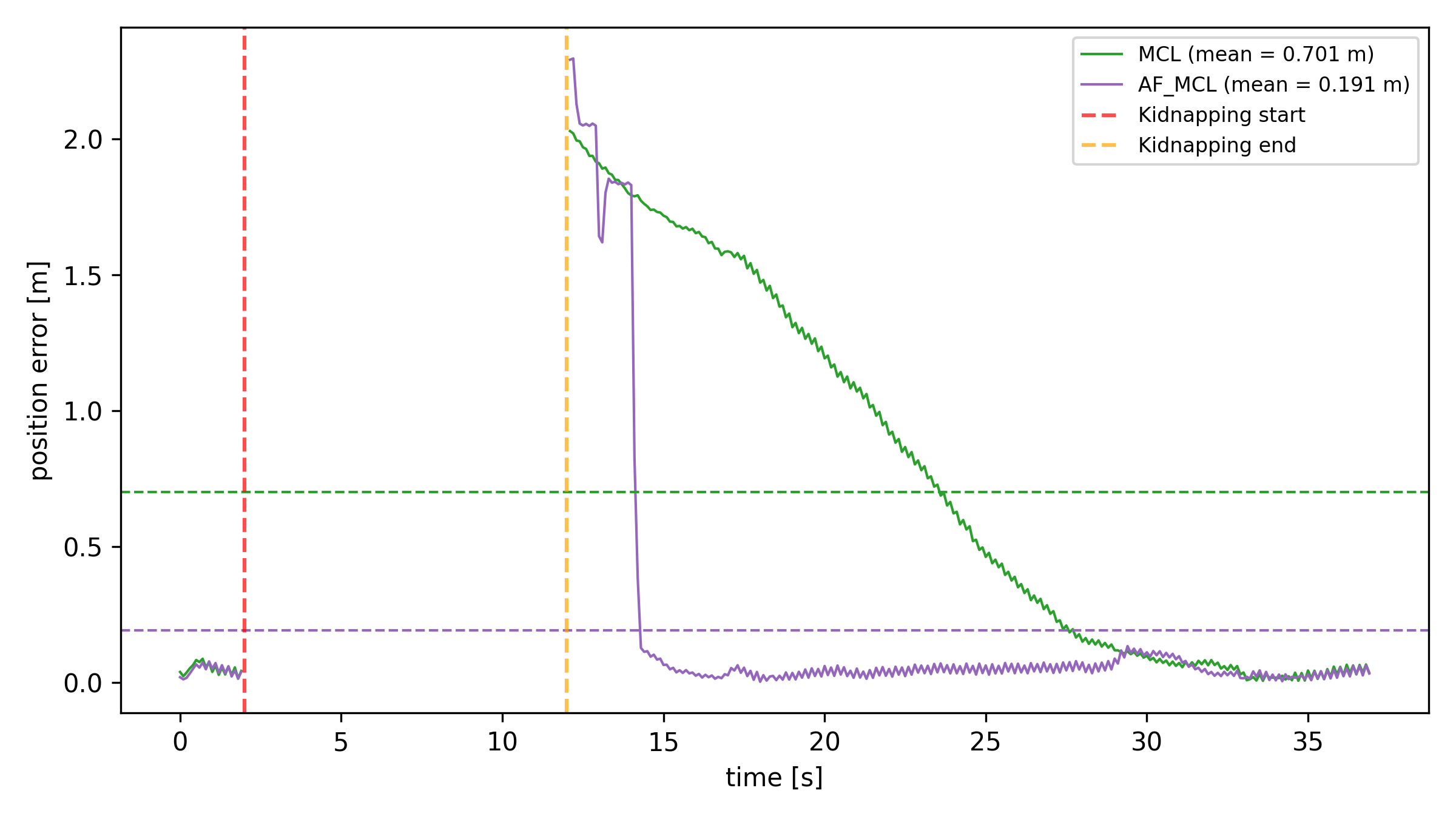}
      \caption{}
      \label{fig:kidnapping_err}
    \end{subfigure}\\[1ex]
    \begin{subfigure}{0.48\linewidth}
      \centering
      \includegraphics[width=\linewidth]{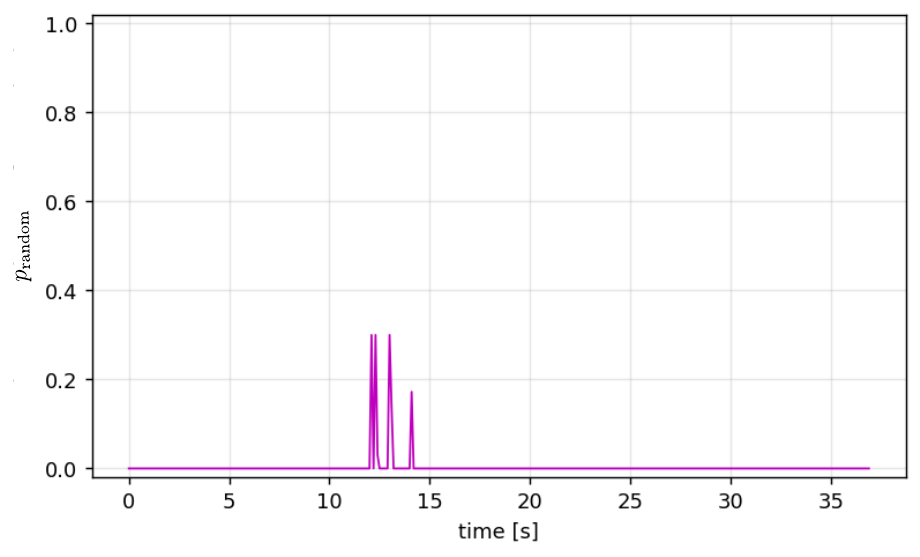}
      \caption{}
      \label{fig:kidnapping_prandom}
    \end{subfigure}\hfill
    \begin{subfigure}{0.48\linewidth}
      \centering
      \includegraphics[width=\linewidth]{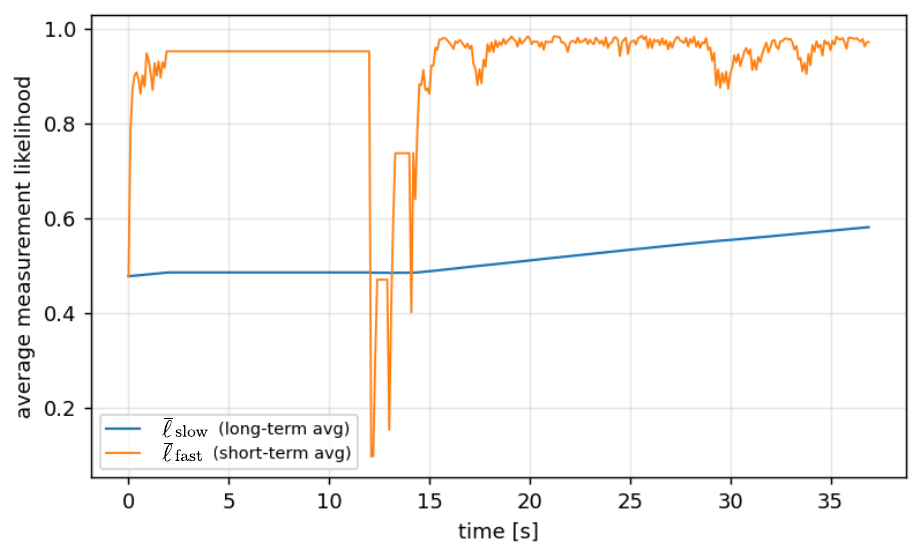}
      \caption{}
      \label{fig:kidnapping_weights}
    \end{subfigure}
  \end{subfigure}
  \caption{Kidnapping scenario: (a) Trajectories and kidnapping window. (b) Position error over time. (c) $p_{\mathrm{random}}$ particle injection. (d) $\bar{\ell}_{\mathrm{slow}}$ and $\bar{\ell}_{\mathrm{fast}}$ triggering recovery.}
  \label{fig:simulation_comparison_kidnapping}
\end{figure}
\subsection{Runtime vs. Number of Particles}
\label{sec:runtime_scaling}
To examine the complexity analysis of Proposition~\ref{prop:complexity}, we instrumented the AF-MCL update by stage and measured runtime, position error, and $\sigma$ for $N \in \{25, 50, 100, 200, 400\}$ along an S-shaped trajectory in a $6\,\mathrm{m} \times 6\,\mathrm{m}$ environment with two dynamic obstacles (Table~\ref{tab:runtime_particle}).

\begin{table}[t]
    \centering
    \caption{Runtime and position error vs. $N$. Mean stage times in ms: Cls = beam classification (incl. ray casting), Prp = propagation, Wgt = weighting, Rsm = resampling, Tot = total update time (mean~$\pm$~std, ms); Err = position error (m); $\sigma$ = online uncertainty (m).}
    \label{tab:runtime_particle}
    \small
    \begin{tabularx}{\linewidth}{
       >{\hsize=0.05\hsize}X
      >{\hsize=0.1\hsize}X
      >{\hsize=0.1\hsize}X
      >{\hsize=0.1\hsize}X
      >{\hsize=0.1\hsize}X
      >{\hsize=0.2\hsize}X
      >{\hsize=0.2\hsize}X
      }
        \toprule
        $N$ & Cls & Prp & Wgt & Rsm & Tot & Err\\
        \midrule
        25  & 1.79 & 0.08 & 0.78 & 0.18 & 2.95~$\pm$~0.92 & 0.049~$\pm$~0.036 \\
        50  & 2.04 & 0.09 & 0.92 & 0.29 & 3.48~$\pm$~1.10 & 0.050~$\pm$~0.045 \\
        100 & 1.95 & 0.09 & 0.90 & 0.46 & 3.53~$\pm$~0.81 & 0.048~$\pm$~0.023 \\
        200 & 2.09 & 0.11 & 1.07 & 0.83 & 4.25~$\pm$~1.16 & 0.046~$\pm$~0.028 \\
        400 & 2.02 & 0.13 & 1.08 & 1.51 & 4.88~$\pm$~0.76 & 0.047~$\pm$~0.025 \\
        \bottomrule
    \end{tabularx}
\end{table}

The measurements are consistent with Proposition~\ref{prop:complexity}. A $16\times$ increase in $N$ raises the total update time by only about $65\%$. The classification time, of which ray casting accounts for about $97\%$, stays within 1.8--2.1\,ms for all $N$, as expected from its $N$-independent cost $\mathcal{O}(MC_{\mathrm{ray}})$. The particle-dependent stages grow with $N$: resampling increases about $8\times$, weighting about $1.4\times$, and propagation remains negligible. Classification is the largest single cost for every $N$, but its share of the total falls from about $61\%$ at $N=25$ to $41\%$ at $N=400$, where resampling (31\%) becomes comparable. Hence, dynamic-beam handling does not add per-particle ray casting, in contrast to a particle-wise scheme with cost $\mathcal{O}(NMC_{\mathrm{ray}})$.

Position error shows no meaningful dependence on $N$ in this range, since the differences are well within one standard deviation, so $N=25$--$100$ already gives sufficient accuracy at lower cost in this setup.

\subsection{Validation of the Geometric Analysis} \label{sec:observability_validation} 
To illustrate Propositions~\ref{prop:observability} and~\ref{prop:error_bound}, we evaluate $\sigma_{\min}(H_{Q_t})$ along an S-shaped trajectory in the $6\,\mathrm{m}\times6\,\mathrm{m}$ arena. Four dynamic obstacles are placed so that beam rejection periodically leaves only the front and back sensors. Their range gradients are anti-parallel, and therefore Proposition~\ref{prop:observability} predicts $\operatorname{rank}H_{Q_t}=1$ and $\sigma_{\min}(H_{Q_t})=0$ whenever only this pair is retained. Thus, filtering corrupted measurements can simultaneously reduce the geometric information available for localization. 

Fig.~\ref{fig:sigma_min_vs_error} shows $\sigma_{\min}(H_{Q_t})$ and the position error over time. The predicted degeneracy occurs around $t=3.8$--$3.9$~s, $7.6$--$8.5$~s, and $18.2$--$18.5$~s, where $\sigma_{\min}$ reaches zero. The largest error increase, reaching approximately $0.19$~m, coincides with the $7.6$--$8.5$~s interval, consistent with the weaker geometric constraints identified by Proposition~\ref{prop:error_bound}. However, not every loss of instantaneous rank produces a large error, since AF-MCL also carries information from the prior belief and motion model. The experiment therefore illustrates the conditioning of the retained measurements rather than a direct relation between $\sigma_{\min}(H_{Q_t})$ and the AF-MCL estimation error.

\begin{figure}[t]
    \centering
    \includegraphics[width=0.6\linewidth]{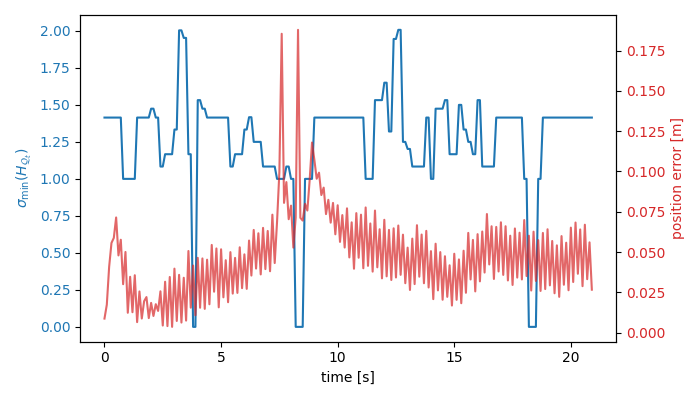}
    \caption{$\sigma_{\min}(H_{Q_t})$ (blue) and position error (red) along the S-shaped trajectory with dynamic obstacles. Drops in $\sigma_{\min}$ from beam rejection coincide with elevated localization error.}
    \label{fig:sigma_min_vs_error}
\end{figure}

\subsection{Real-World Experimental Results}
\label{subsec:Real_World_Experimental_Results}
The proposed method was validated using a Crazyflie~2.1 equipped with a Flow Deck v2 and a Multi-ranger Deck, which together provide four-directional range measurements (Fig.~\ref{fig:crazyflie_radio}a). A Crazyradio PA (Fig.~\ref{fig:crazyflie_radio}b) handled communication with the host PC, and eight Qualisys cameras supplied ground-truth position (Fig.~\ref{fig:crazyflie_radio}c). The host PC runs Ubuntu~22.04.5 LTS on an Intel Core i7 CPU with 16~GB RAM. Velocity commands are transmitted every 100~ms; the official Crazyflie firmware\footnote{\url{https://github.com/bitcraze/crazyflie-firmware}} converts these commands into motor outputs via PID control on the onboard STM32F405 microcontroller. The proposed localization method runs externally on the host PC; onboard execution on the Crazyflie was not evaluated. Prior to the real-world trials, the setup was first validated in Gazebo Garden using CrazySim~\cite{LlanesICRA2024}.

\begin{figure}[t]
  \centering
  \begin{subfigure}{0.24\linewidth}
    \centering
    \includegraphics[width=\linewidth, height=\linewidth]{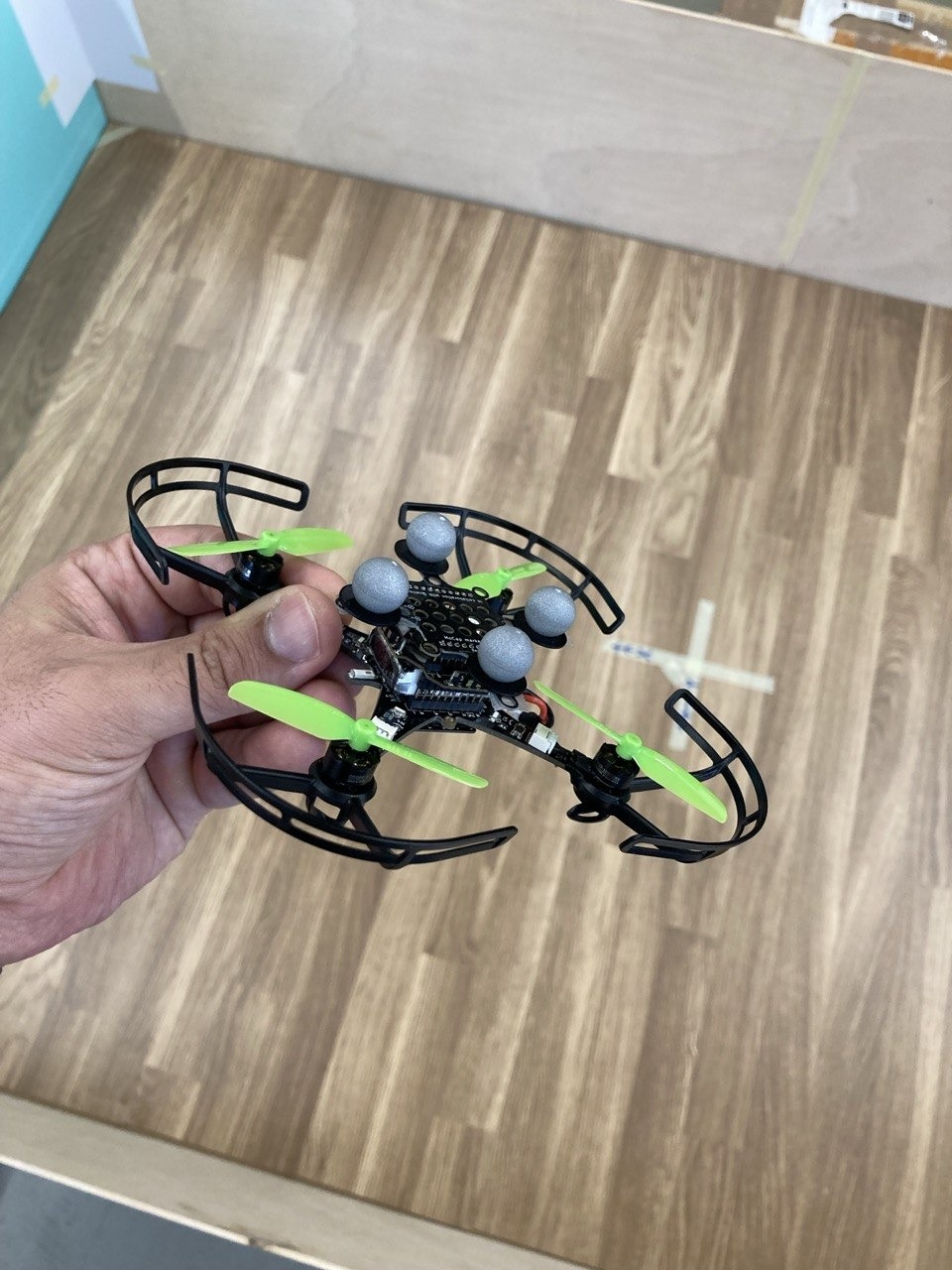}
    \caption{}
  \end{subfigure}
  \begin{subfigure}{0.24\linewidth}
    \centering
    \includegraphics[width=\linewidth, height=\linewidth]{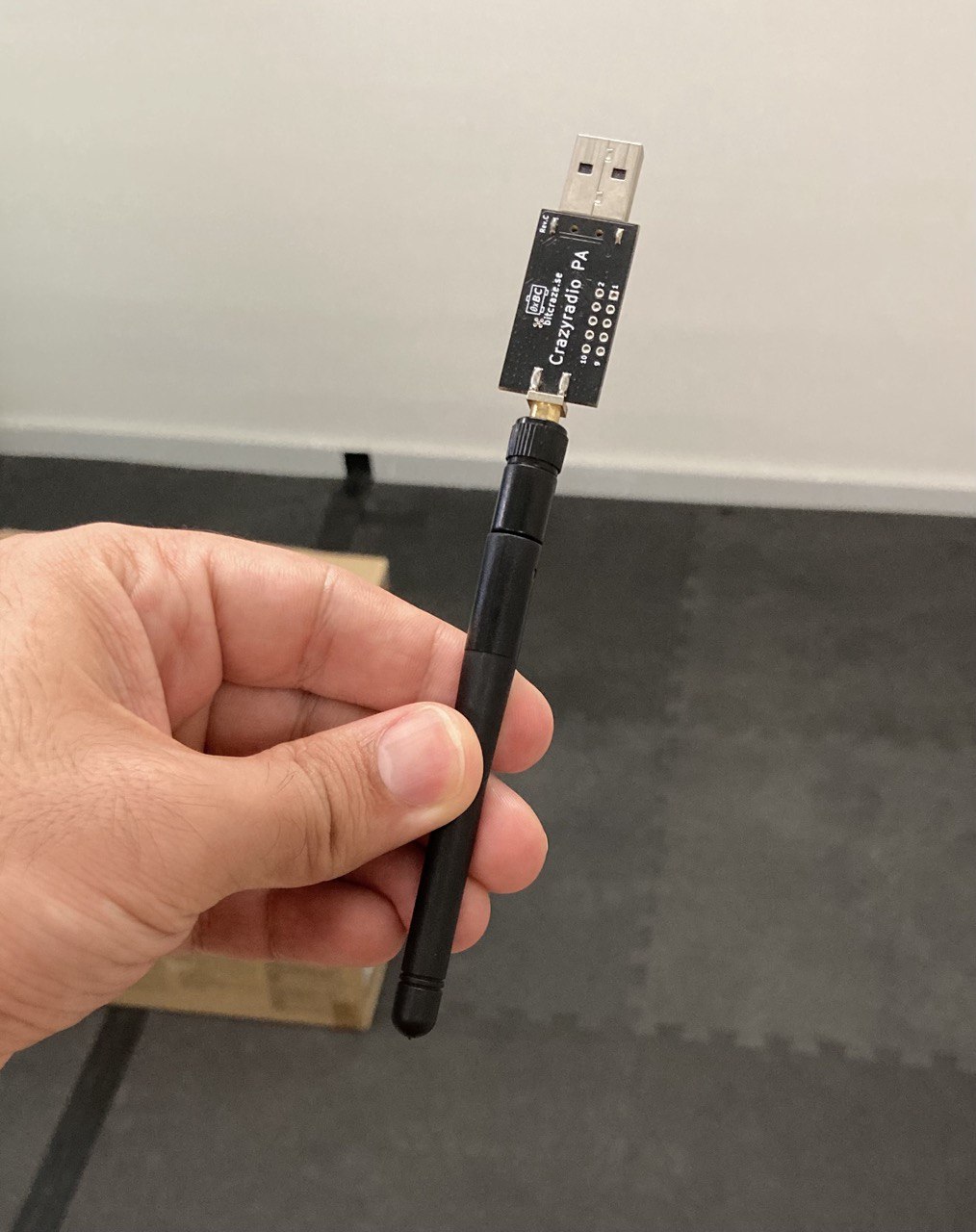}
    \caption{}
  \end{subfigure}
  \begin{subfigure}{0.48\linewidth}
    \centering
    \includegraphics[width=\linewidth, height=0.5\linewidth]{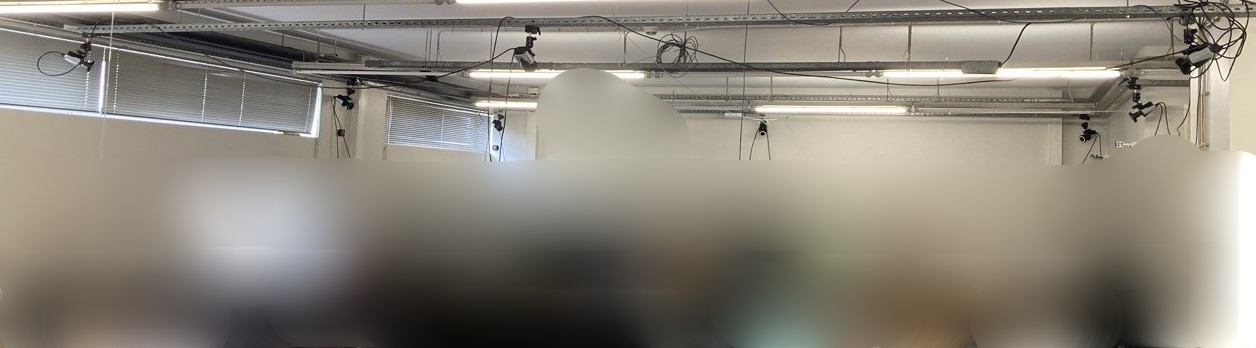}
    \caption{}
  \end{subfigure}
  \caption{(a) Crazyflie~2.1 with Flow Deck v2 and Multi-ranger Deck. (b) Crazyradio PA used for host-PC communication. (c) Eight Qualisys cameras supplied ground-truth position.}
  \label{fig:crazyflie_radio}
\end{figure}

Six scenarios were evaluated in a $1.63\times2.10~\mathrm{m}^2$ environment: S-shaped and square trajectories, each flown with no obstacles, with static obstacles only, and with combined static and dynamic obstacles. The four obstacle-containing configurations are shown in Fig.~\ref{fig:env_testing}. Table~\ref{tab:amcl_results} reports the mean position error and mean runtime for all six scenarios.

\begin{figure*}[t]
  \centering
  \begin{subfigure}{0.24\linewidth}
    \centering
    \includegraphics[width=\linewidth, height=0.5\linewidth]{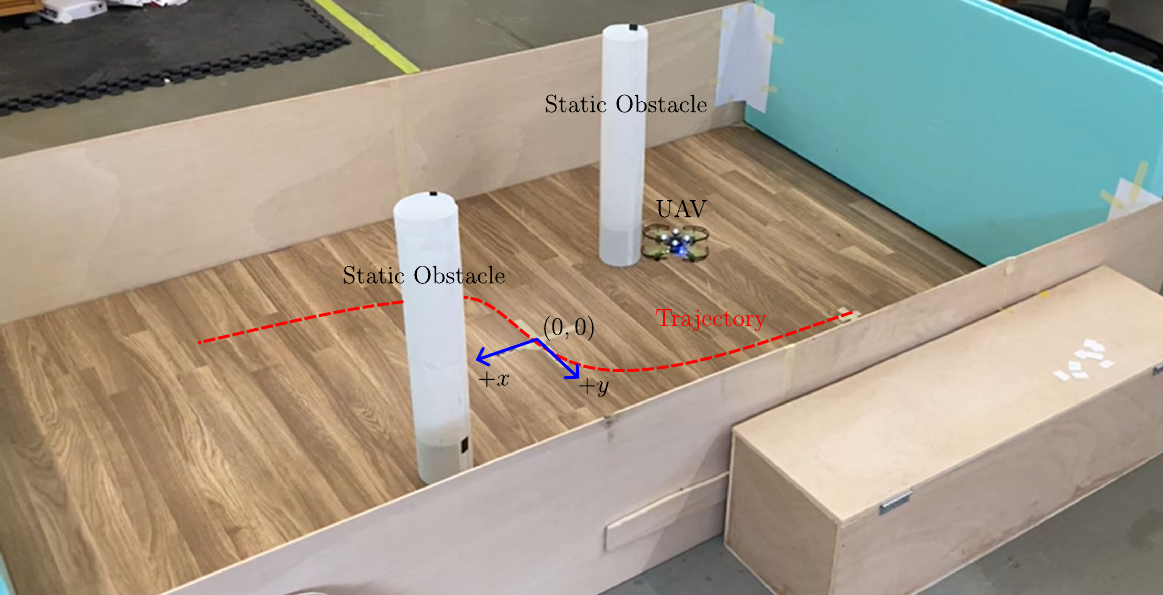}
    \caption{}
  \end{subfigure}
  \begin{subfigure}{0.24\linewidth}
    \centering
    \includegraphics[width=\linewidth, height=0.5\linewidth]{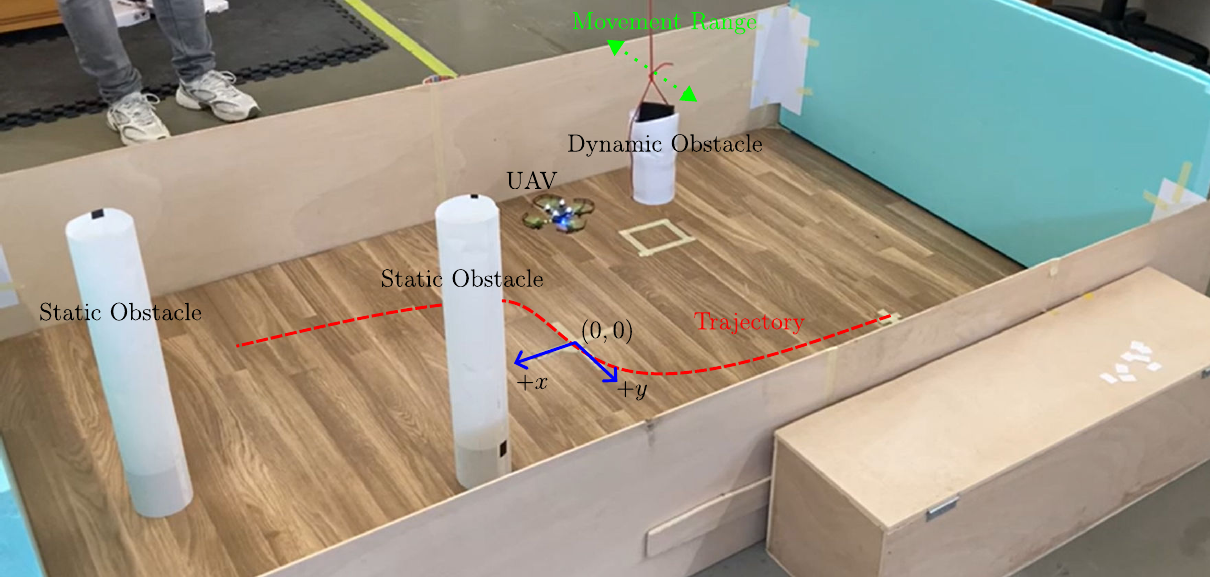}
    \caption{}
  \end{subfigure}
  \begin{subfigure}{0.24\linewidth}
    \centering
    \includegraphics[width=\linewidth, height=0.5\linewidth]{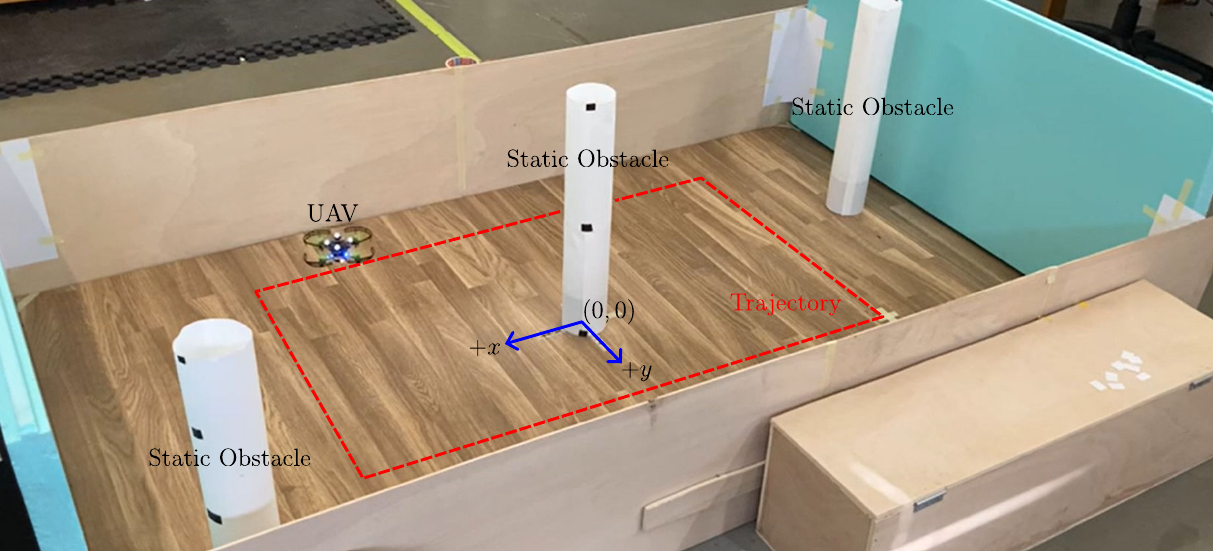}
    \caption{}
  \end{subfigure}
  \begin{subfigure}{0.24\linewidth}
    \centering
    \includegraphics[width=\linewidth, height=0.5\linewidth]{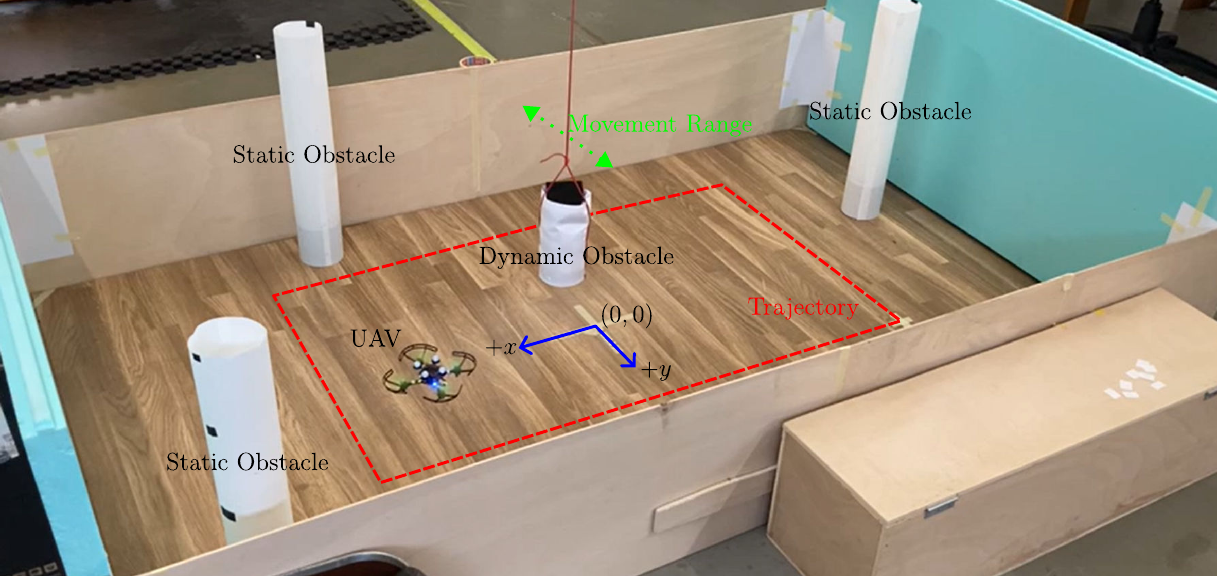}
    \caption{}
  \end{subfigure}
  \caption{Real-world test environments: (a) S-shape trajectory with static obstacles, (b) S-shape trajectory with static and dynamic obstacles, (c) square trajectory with static obstacles, and (d) square trajectory with static and dynamic obstacles.}
  \label{fig:env_testing}
\end{figure*}

\begin{table}[t]
\centering
\caption{AF-MCL localization performance }
\label{tab:amcl_results}
\begin{tabularx}{\linewidth}{
   >{\hsize=0.15\hsize}X
  >{\hsize=0.45\hsize}X
  >{\hsize=0.35\hsize}X
  >{\hsize=0.3\hsize}X
  >{\hsize=0.3\hsize}X
  }
\toprule
\textbf{Traj.} & \textbf{Scenario} & \textbf{Mean Pos.} & \textbf{Mean} & \textbf{Mean} $\sigma$ \\
 &  & \textbf{Error} [m] & \textbf{Runtime}[ms] & [m] \\
\midrule
\multirow{4}{*}{S-shape}
 & No Obstacle       & 0.067~$\pm$~0.012 & 1.43~$\pm$~0.24 & 0.058 \\
 & Static Obstacles  & 0.070~$\pm$~0.009 & 1.21~$\pm$~0.02 & 0.061 \\
 & Static\&Dynamic Obstacles & 0.064~$\pm$~0.004 & 1.21~$\pm$~0.01 & 0.068 \\
\midrule
\multirow{4}{*}{Square}
 & No Obstacle       & 0.068~$\pm$~0.009 & 1.26~$\pm$~0.03 & 0.050 \\
 & Static Obstacles  & 0.073~$\pm$~0.006 & 1.21~$\pm$~0.06 & 0.051 \\
 & Static\&Dynamic Obstacles & 0.077~$\pm$~0.004 & 1.13~$\pm$~0.02 & 0.053 \\
\bottomrule
\end{tabularx}
\end{table}


Table~\ref{tab:amcl_results} shows that AF-MCL keeps the mean position error within a narrow range across all six scenarios, so its accuracy is largely insensitive to trajectory shape and obstacle configuration. For the S-shape trajectory the error is essentially flat across obstacle configurations, and for the square trajectory it grows slightly as obstacles are added, with a total increase under 1~cm. The online uncertainty $\sigma$ has the same magnitude as the error, but it is lower than the error in five of the six scenarios, most clearly for the square trajectory. This is expected, since $\sigma$ reflects only the spread of the particle set and does not include the map and ground-truth errors discussed below. The mean runtime stays well within the 100~ms command interval on the host computer, supporting real-time operation and motivating future onboard implementation. The runtime was measured on the host PC, so it does not demonstrate execution on the Crazyflie flight controller.


Figs.~\ref{fig:real_plots}a--d compare the AF-MCL estimate with ground truth for the four obstacle-containing scenarios (S-shape and square, with static obstacles and with static and dynamic obstacles, respectively). In all cases the estimated path follows the ground truth without visible divergence or drift, even near obstacles. Figs.~\ref{fig:real_plots}e--h show that the position error stays low and stable over time rather than growing, so AF-MCL does not accumulate drift. The square trajectory with dynamic obstacles (h) shows the highest error, consistent with Table~\ref{tab:amcl_results}.

The reported errors are also affected by two sources of measurement uncertainty outside the localization algorithm, each with an estimated engineering bound of approximately $\pm2$~cm (not a measured RMS error or standard deviation, and not combined here). First, the occupancy grid map was built manually from measurements of the environment size and obstacle positions, which introduces map errors. The range sensors are also sensitive to surface color, and darker walls or obstacles produce noisier readings. Second, the Qualisys ground truth can degrade when the UAV is hidden behind obstacles. Because the surrounding walls are tall, fewer cameras see the markers at these positions, which reduces tracking accuracy. These errors contribute directly to the reported position error, so Table~\ref{tab:amcl_results} should be read as an accuracy estimate within these bounds.
\begin{figure*}
  \centering
  \begin{subfigure}{0.24\linewidth}
    \centering
    \includegraphics[width=\linewidth]{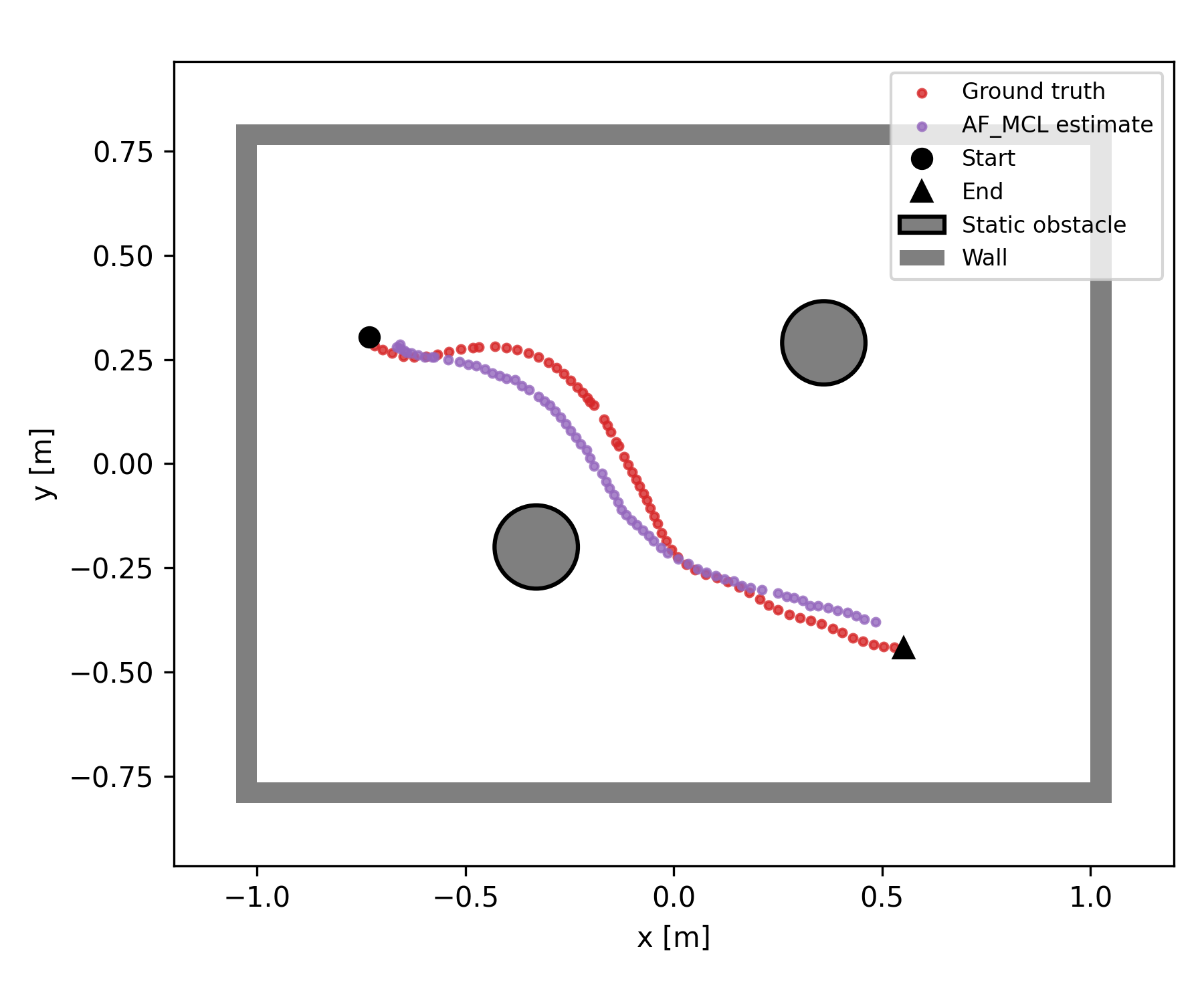}
    \caption{}
  \end{subfigure}
  \begin{subfigure}{0.24\linewidth}
    \centering
    \includegraphics[width=\linewidth]{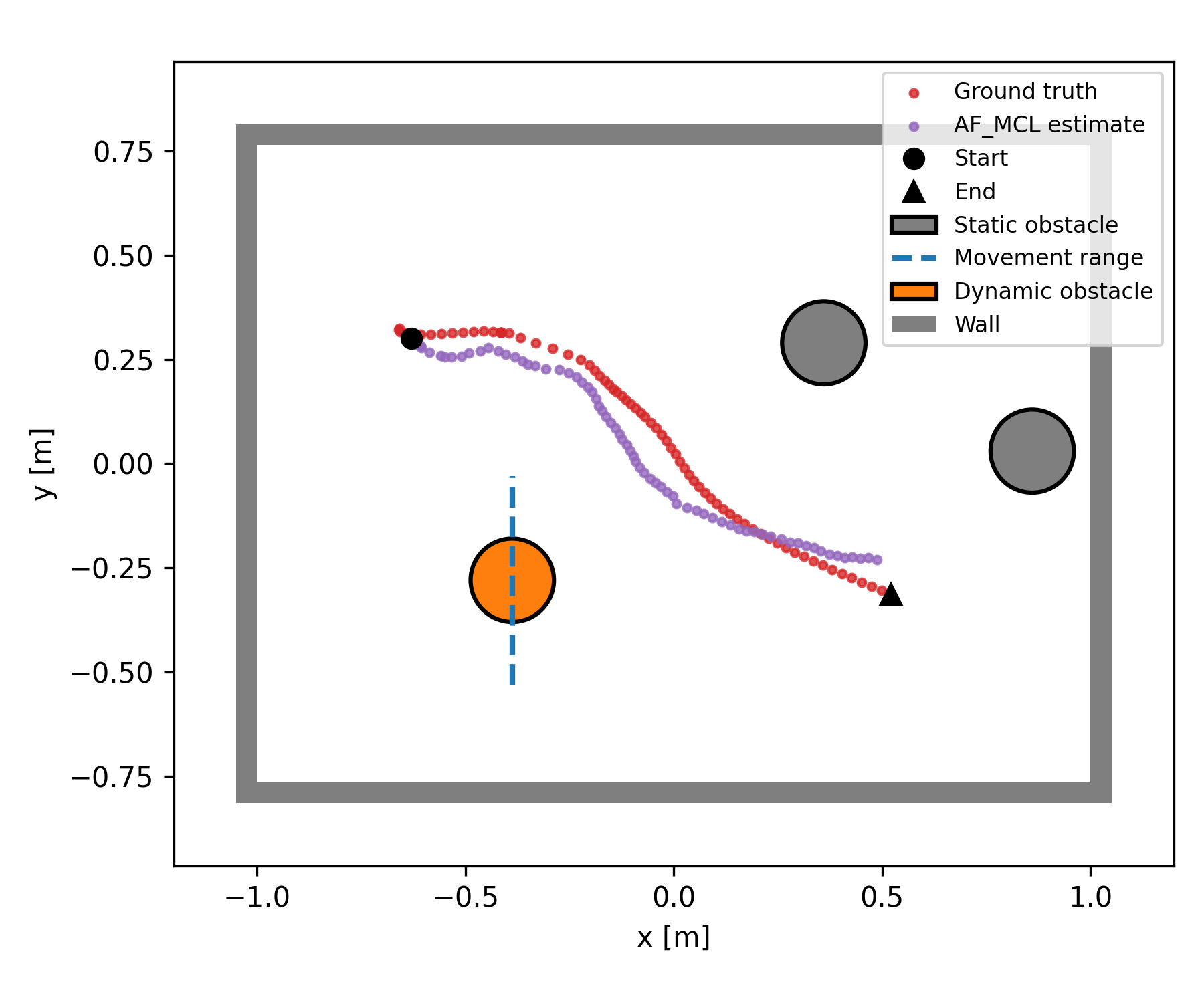}
    \caption{}
  \end{subfigure}
  \begin{subfigure}{0.24\linewidth}
    \centering
    \includegraphics[width=\linewidth]{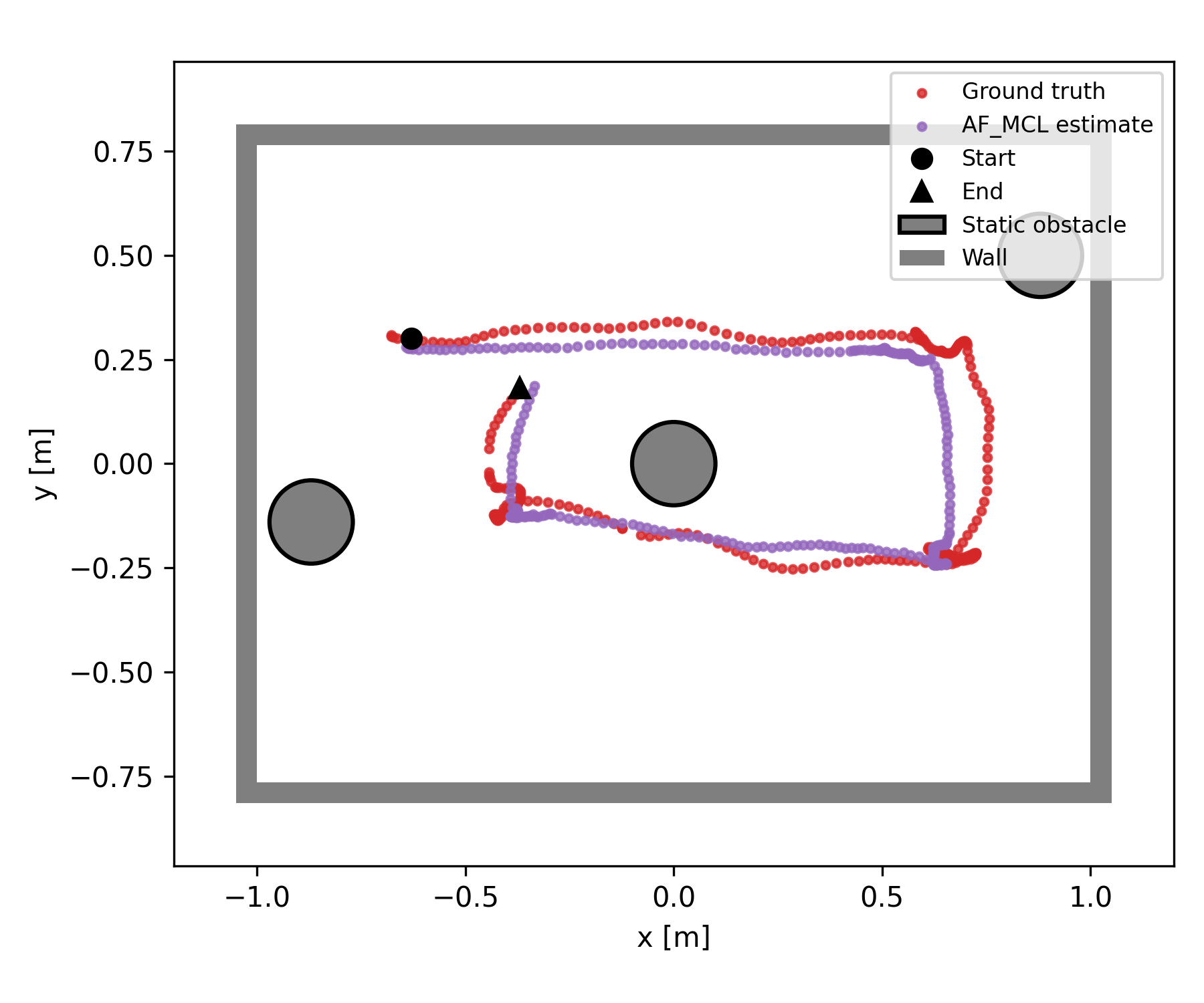}
    \caption{}
  \end{subfigure}
  \begin{subfigure}{0.24\linewidth}
    \centering
    \includegraphics[width=\linewidth]{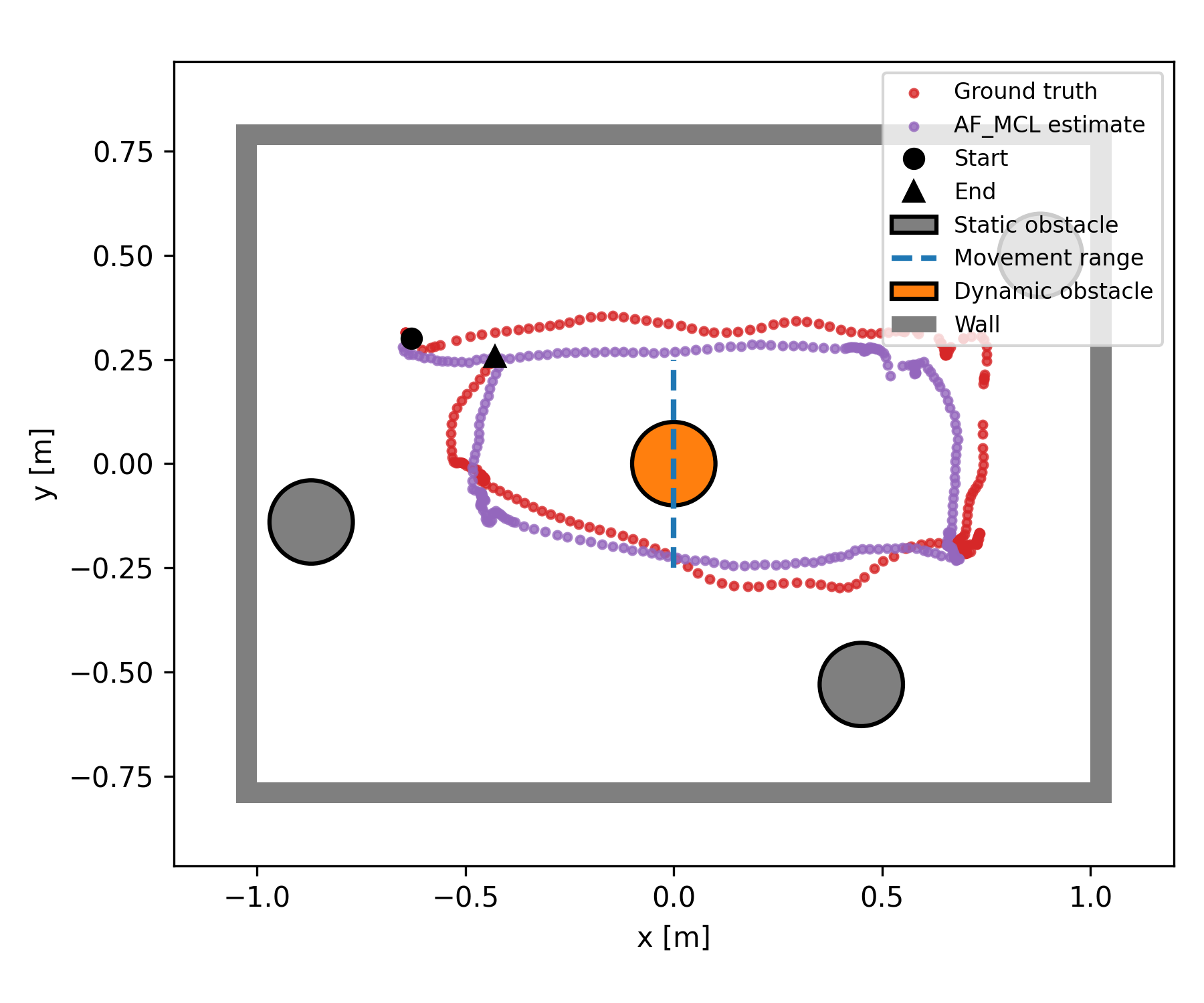}
    \caption{}
  \end{subfigure}
  \begin{subfigure}{0.24\linewidth}
    \centering
    \includegraphics[width=\linewidth]{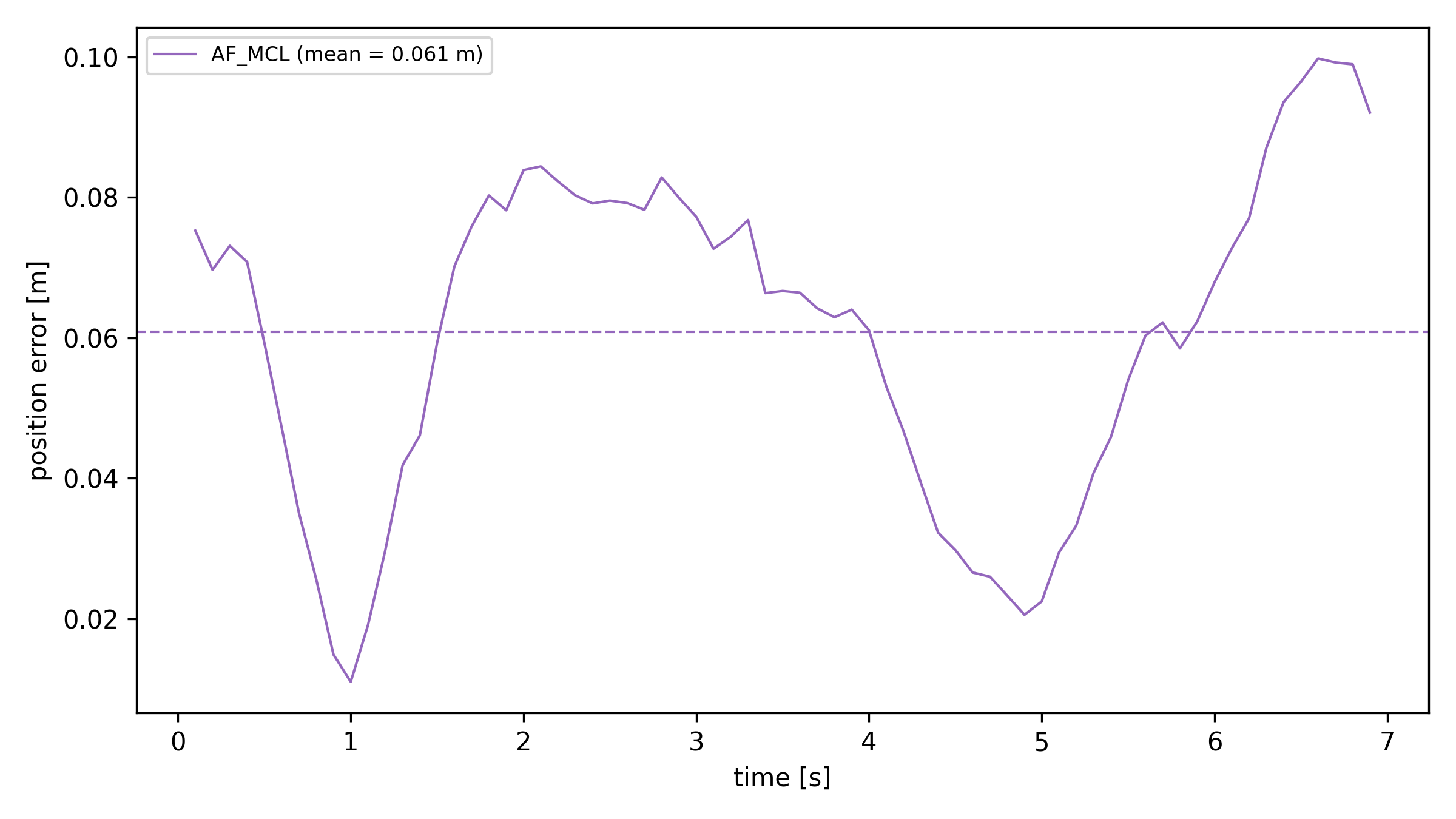}
    \caption{}
  \end{subfigure}
  \begin{subfigure}{0.24\linewidth}
    \centering
    \includegraphics[width=\linewidth]{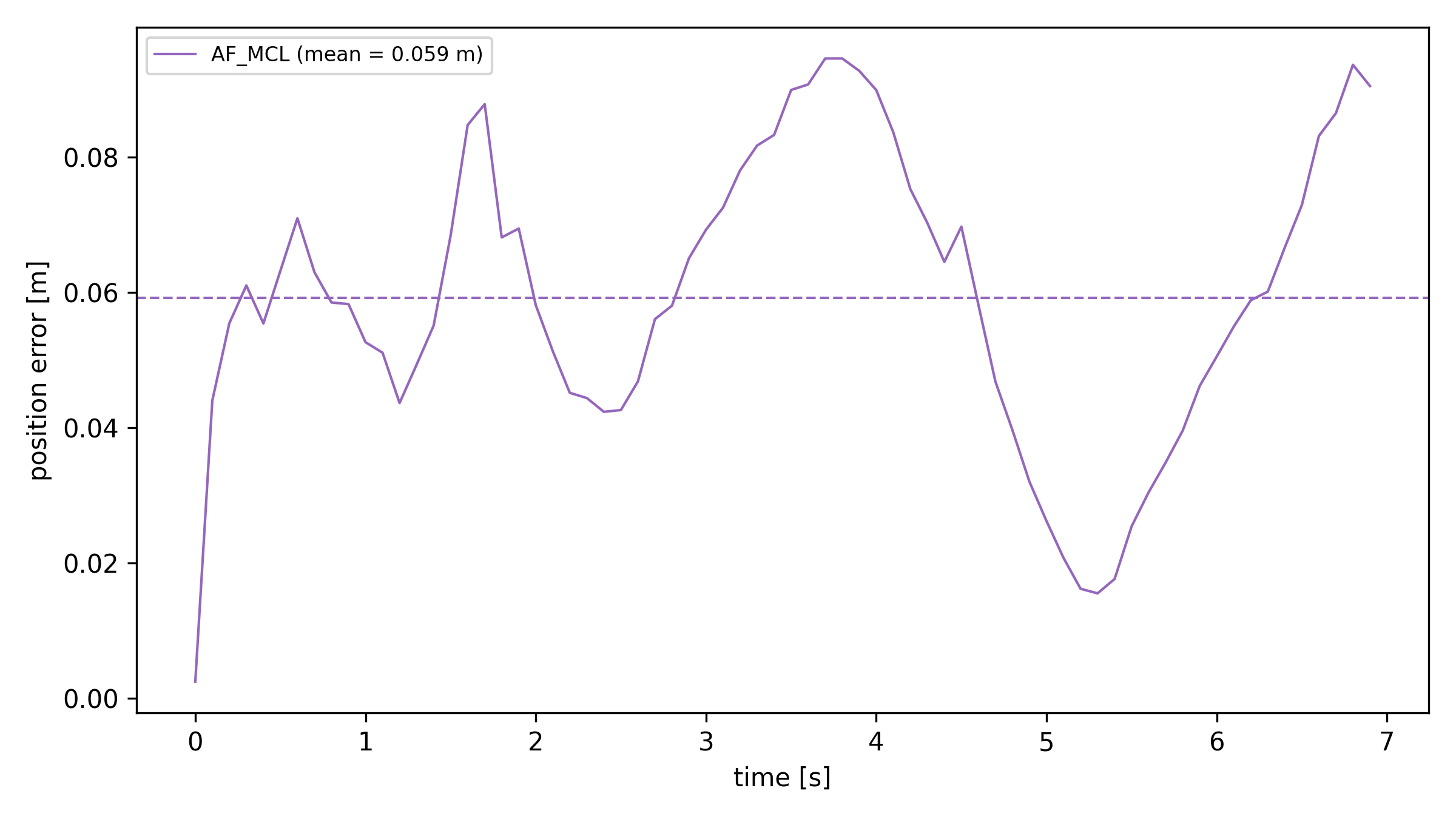}
    \caption{}
  \end{subfigure}
  \begin{subfigure}{0.24\linewidth}
    \centering
    \includegraphics[width=\linewidth]{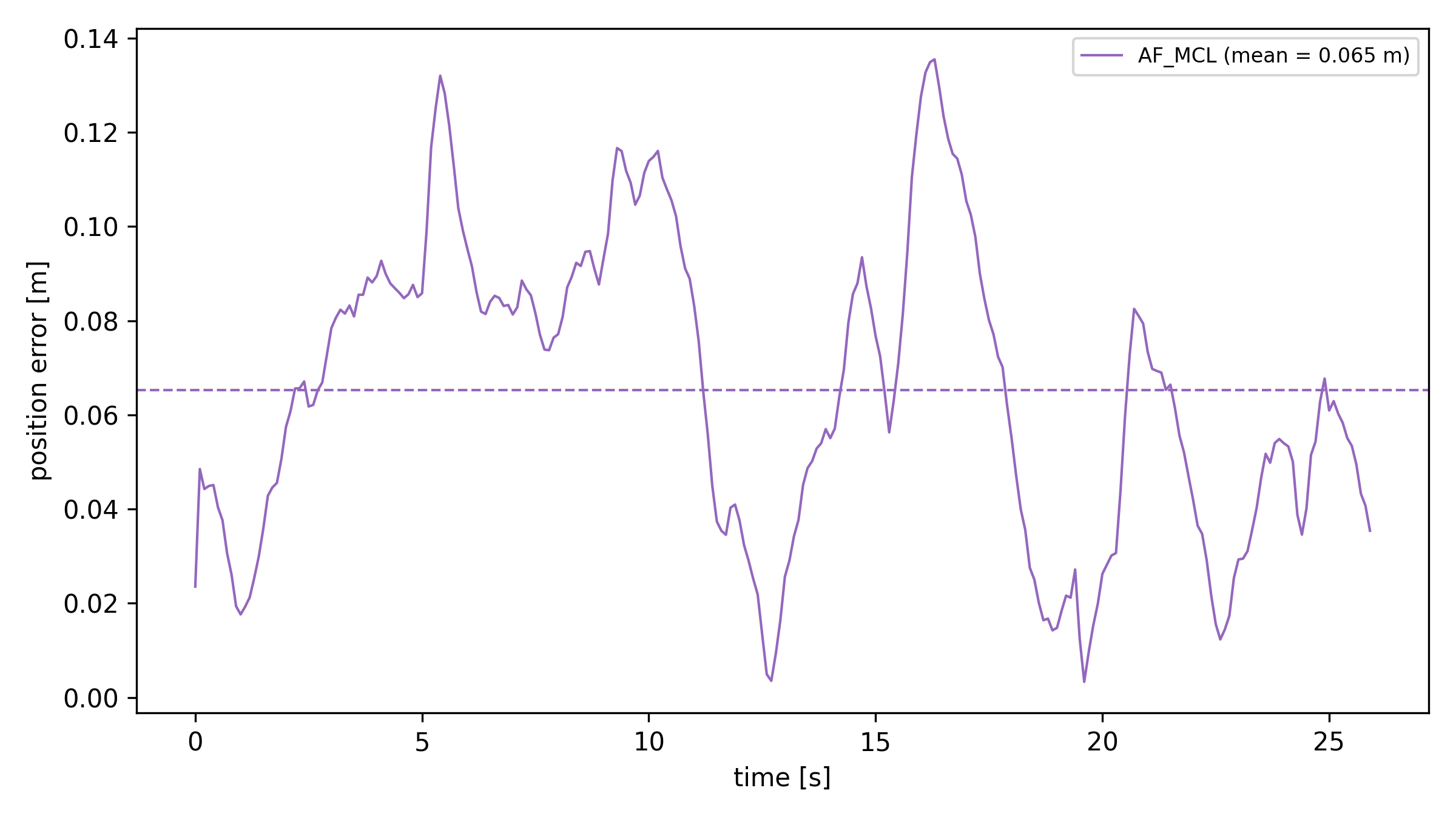}
    \caption{}
  \end{subfigure}
  \begin{subfigure}{0.24\linewidth}
    \centering
    \includegraphics[width=\linewidth]{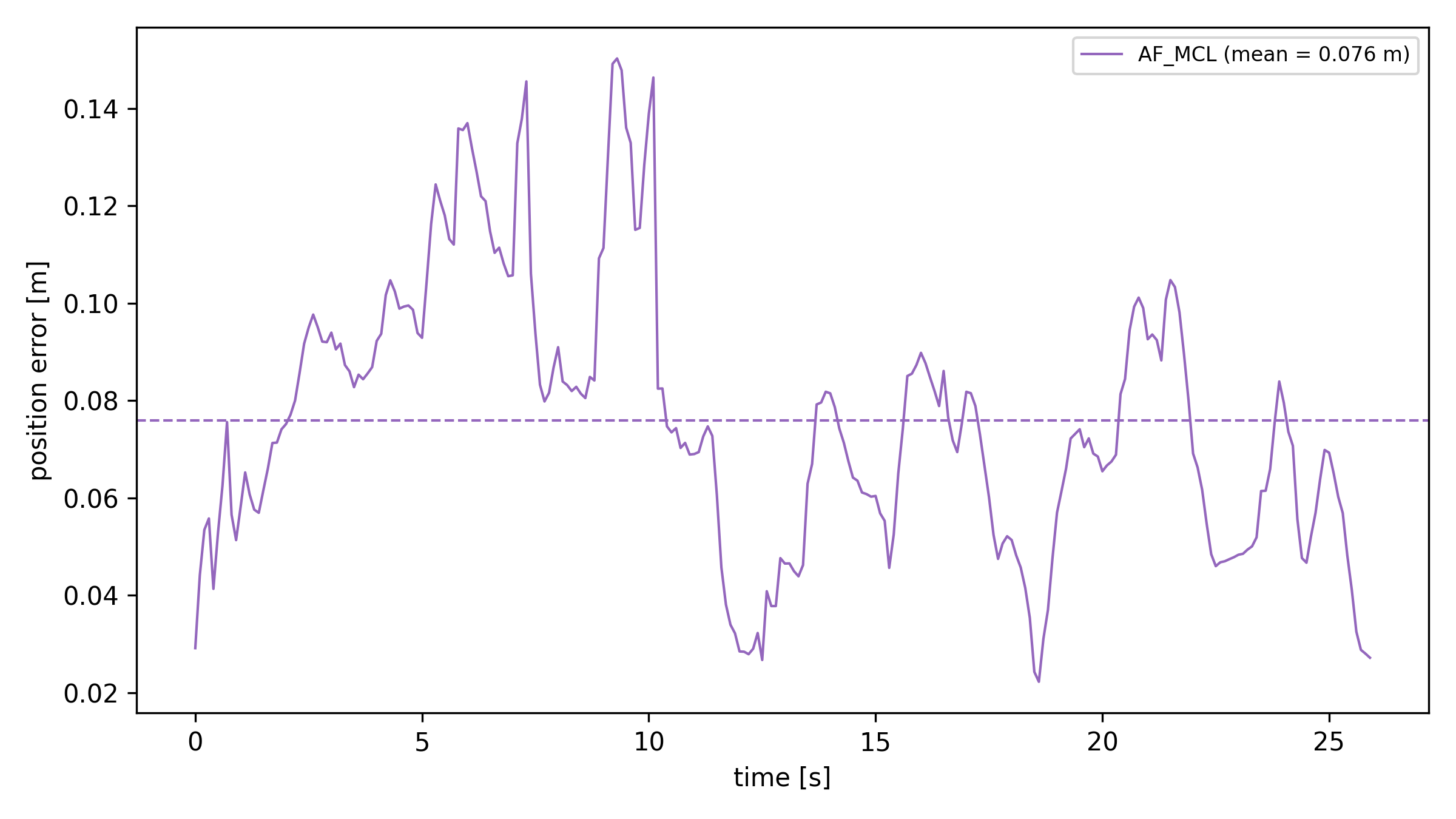}
    \caption{}
  \end{subfigure}
  \caption{Real-world scenario results. (a)--(d) AF-MCL position estimate versus ground truth for: (a) S-shape trajectory with static obstacles; (b) S-shape trajectory with static and dynamic obstacles; (c) square trajectory with static obstacles; (d) square trajectory with static and dynamic obstacles. (e)--(h) Corresponding position error over time for (a)--(d), with mean errors of 0.061~m, 0.059~m, 0.065~m, and 0.076~m, respectively.}
  \label{fig:real_plots}
\end{figure*}


To test the simulation findings on physical data, the logged real-world measurements (ranges, velocities, yaw, and ground truth) were replayed offline through MCL, A-MCL, F-MCL, and AF-MCL with identical parameters and map, using 5 random seeds per flight. The replay is open-loop, since the logged ranges come from flights run with AF-MCL. For the square trajectory with static and dynamic obstacles, standard MCL degrades to 0.096\,m mean error, while beam filtering reduces it by about 23\% (F-MCL and AF-MCL, 0.074\,m). A-MCL is the worst (0.186\,m), because corrupted likelihoods trigger false particle injection. Standard MCL's $\sigma$ stays at 0.047\,m despite its larger error, so it is overconfident, whereas A-MCL's $\sigma$ inflates (0.169\,m) with its error. For F-MCL and AF-MCL, $\sigma$ (0.048 and 0.050\,m) stays at the scale of the error. This agrees with the simulation ablation: filtering, not augmentation, is what protects accuracy under dynamic obstacles.
\section{Conclusion} \label{conclusion}
We presented AF-MCL for indoor UAV localization using only four single-point range sensors and measured heading. AF-MCL combines dynamic-beam filtering with adaptive particle injection for recovery from localization failures, while performing beam classification independently of the particle count. The particle distribution additionally provides an online measure of localization uncertainty. Real-world experiments achieved mean errors of 6.4--7.7~cm, while simulations in $6\times6$~m and $10\times10$~m environments achieved 4.5--5.2~cm, outperforming standard MCL under dynamic obstacles. The geometric analysis showed that beam rejection can leave the retained measurements poorly conditioned or rank deficient, increasing their sensitivity to measurement perturbations. Future work will exploit this information for geometry-aware filtering and uncertainty adaptation, and extend the approach to localization without a prior map.
\bibliographystyle{ieeetr}
\bibliography{ref}
\end{document}